\documentclass[runningheads]{llncs}
\usepackage[round]{natbib}
\usepackage[english]{babel}
\usepackage[hidelinks]{hyperref}
\usepackage{amsfonts,amsmath,amssymb}
\usepackage{graphicx}

\usepackage{xcolor}
\definecolor{couleurA}{HTML}{00A08A}
\definecolor{couleurB}{HTML}{F2AD00}
\definecolor{couleurzero}{HTML}{5BBCD6}
\definecolor{couleurun}{HTML}{FF0000}
\definecolor{couleurwith}{HTML}{046C9A}
\definecolor{couleurwithout}{HTML}{C93312}
\newcommand{\code}[1]{\text{\texttt{#1}}}

\usepackage[textwidth=3cm, textsize=tiny]{todonotes}
\newcommand{\indep}{\perp \!\!\! \perp}

\begin{document}
\title{Mitigating Discrimination in Insurance\\with Wasserstein Barycenters}
\titlerunning{Mitigating Discrimination in Insurance with Wasserstein Barycenters}

\author{Arthur Charpentier\inst{1*} \and
Fran\c{c}ois Hu\inst{2} \and
Philipp Ratz\inst{1}}

\authorrunning{Charpentier, Hu \& Ratz}

\institute{Universit\'e du Qu\'ebec \`a Montr\'eal \and
Universit\'e de Montr\'eal\\
${}^{*}$\email{charpentier.arthur@uqam.ca}}%
\maketitle              % typeset the header of the contribution
\begin{abstract}
The insurance industry is heavily reliant on predictions of risks based on characteristics of potential customers. Although the use of said models is common, researchers have long pointed out that such practices perpetuate discrimination based on sensitive features such as gender or race. Given that such discrimination can often be attributed to historical data biases, an elimination or at least mitigation is desirable. With the shift from more traditional models to machine-learning based predictions, calls for greater mitigation have grown anew, as simply excluding sensitive variables in the pricing process can be shown to be ineffective. In this article, we first investigate why predictions are a necessity within the industry and why correcting biases is not as straightforward as simply identifying a sensitive variable. We then propose to ease the biases through the use of Wasserstein barycenters instead of simple scaling. To demonstrate the effects and effectiveness of the approach we employ it on real data and discuss its implications.
\end{abstract}

\section{Introduction and motivation}
% According to \cite{dictionary2022merriam}, "{\em discrimination is the act, practice, or an instance of separating or distinguishing categorically rather than individually.}" CITER \cite{schauer2006profiles}

\subsection{Insurance and discrimination, an ill-posed problem}

\cite{avraham2017discrimination} explained in one short paragraph the dilemma of considering the problem of discrimination in insurance. 
``{\em What is unique about insurance is that even statistical discrimination which by definition is absent of any malicious intentions, poses significant moral and legal challenges. Why? Because on the one hand, policy makers would like insurers to treat their insureds equally, without discriminating based on race, gender, age, or other characteristics, even if it makes statistical sense to discriminate} (...) {\em On the other hand, at the core of
insurance business lies discrimination between risky and non-risky insureds. But riskiness often statistically correlates with the same characteristics policy makers would like to prohibit insurers from taking into account.}" 
%In {\em When Is Discrimination Wrong?}, \cite{hellman2008discrimination} claims that it is necessary to distinguish between {\em illegitimate} and {\em legitimate} discrimination REPRENDRE.  
To illustrate this problem, and highlight why writing about discrimination and insurance can be complicated, consider the example of ``{redlining}". Redlining has been an important issue in the credit and insurance industry in the U.S., which started in the 30's. In 1935, the Federal Home Loan Bank Board (FHLBB) looked at more than 200 cities and created ``{\em residential security maps}" to indicate the level of security for real-estate investments in each surveyed city. On the maps (see Figure \ref{fig:redlining} with a collection of fictitious maps), the newest areas—those considered desirable for lending purposes—were outlined in green and known as ``Type A". ``Type D" neighborhoods were outlined in red and considered the most risky for mortgage support (on the left of Figure \ref{fig:redlining}). Such ``Type D" neighborhoods indeed presented a high proportion of dilapidated (or dis-repaired) buildings (as we can observe on the right of Figure \ref{fig:redlining}). In the 70's, when looking at census data, sociologist noticed that red area, where insurers did not want to offer coverage, were also those with a high proportion of Black people, and following the work John McKnight and Andrew Gordon, ``{redlining}" received more interest. In the right pane of Figure \ref{fig:redlining}, the proportion of Black inhabitants is depicted, which roughly coincides with the redlined areas illustrated in the left pane. Thus, on the hand, it could be seen as ``legitimate" to have a premium for household that could somehow reflect the general conditions of houses. On the other hand, it would be discriminatory to have a premium that is function of the ethnic origin of the policyholder. The neighborhood, the ``unsanitary index" and the proportion of Black people are here strongly correlated variables. Of course, this does not preclude non-Black people living in dilapidated houses outside of the red area, Black people living in wealthy houses inside the red area, etc. When working with aggregated data, it is difficult to disentangle information about sanitary conditions and racial information, to distinguish ``legitimate" and ``non-legitimate" discrimination, as discussed in  \cite{hellman2008discrimination} and 
\cite{fairLB}. 
\begin{figure}[!h]
\includegraphics[width=.32\textwidth]{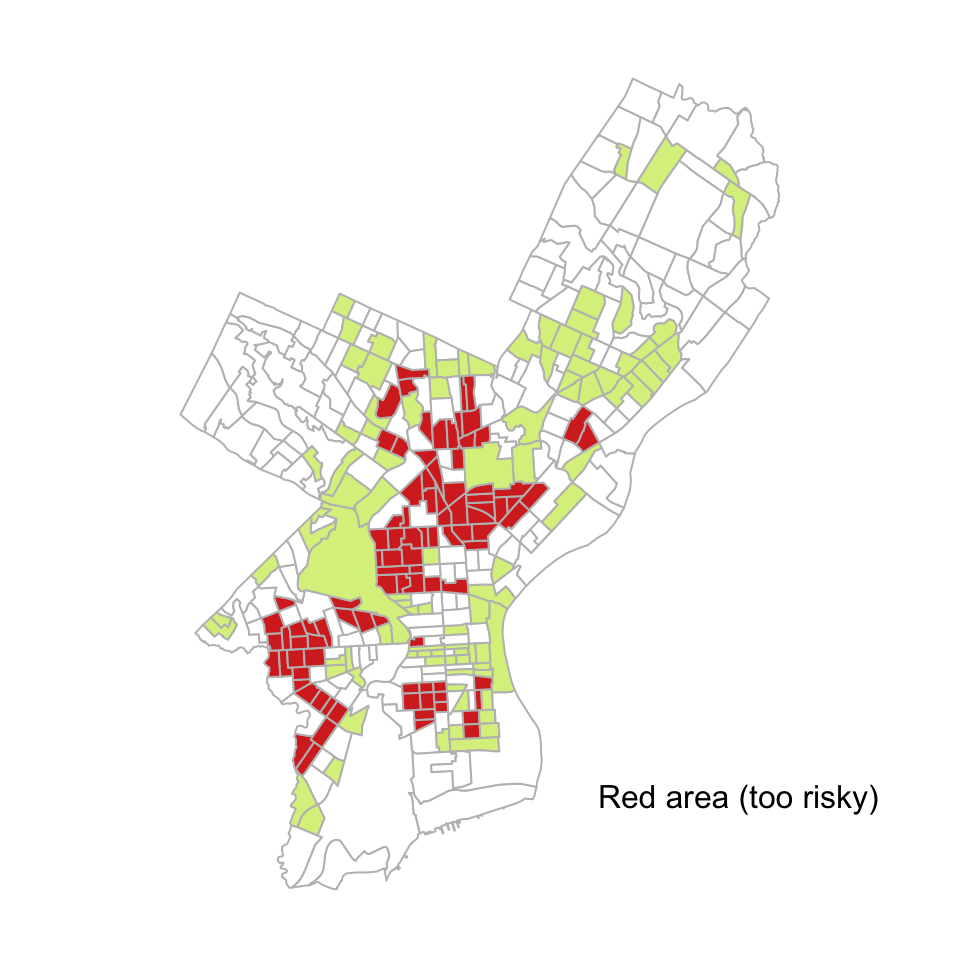}~\includegraphics[width=.32\textwidth]{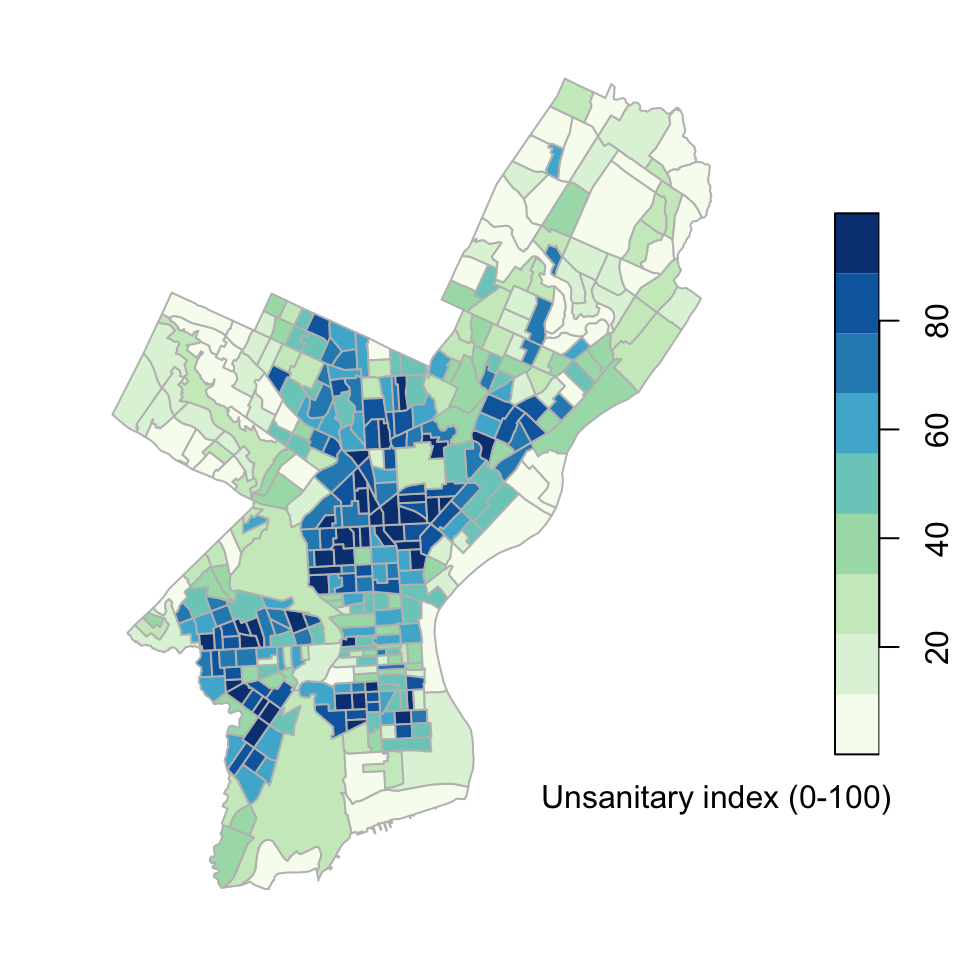}~\includegraphics[width=.32\textwidth]{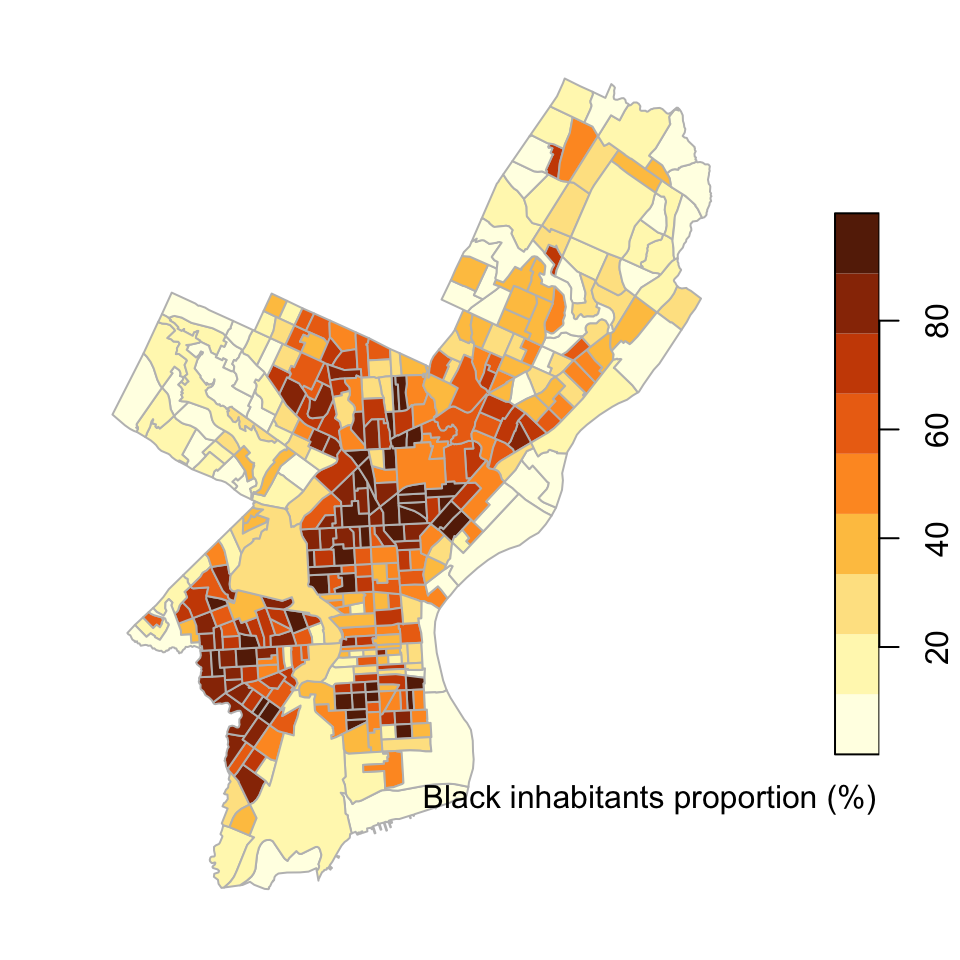} 
\caption{Fictitious maps, (freely) inspired by a Home Owners’ Loan Corporation map from 1937, where red is used to identify neighborhoods in which investment and lending were discouraged, on the left (see \cite{Crossney2016} and \cite{Rhynhart2020}). In the middle, some risk related variable (an fictitious ``unsanitary index") per neighborhood of the city is presented, and on the right, a sensitive variable (the proportion of Black people in the neighborhood, again, freely created).}\label{fig:redlining}
\end{figure}

\subsection{Mitigating discrimination}

Mitigating discrimination is usually seen as paradoxical, because in order to avoid discrimination, one must create another discrimination. More precisely, Supreme Court Justice Harry Blackmun stated, in 1978, ``{\em in order to get beyond racism, we must first take account of race. There is no other way. And in order to treat some persons equally, we must treat them differently}." \index{equality}\index{difference}(\cite{knowlton1978regents}, cited in \cite{lippert2020making}). 
More formally, an argument in favor of affirmative action -- called ``{\em the present- oriented anti-discrimination argument~}" -- is simply that justice requires that we eliminate or at least mitigate (present) discrimination by the best morally permissible means of doing so, which corresponds to affirmative action.
But there are also arguments against affirmative action, corresponding to ``{\em the reverse discrimination objection,}" as defined in \cite{goldman1979justice}: some might consider that there is an absolute ethical constraint against unfair discrimination (including affirmative action). To quote another Supreme Court Justice, in 2007, John G. Roberts of the US Supreme Court submits: ``{\em The way to stop discrimination on the basis of race is to stop discriminating on the basis of race}" (quoted in \cite{turner2015way} and \cite{sabbagh2007equality}). The arguments against affirmative action are usually based on two theoretical moral claims, according to \cite{pojman1998case}. The first denies that groups have moral status (or at least meaningful status). According to this view, individuals are only responsible for the acts they perform as specific individuals and, as a corollary, we should only compensate individuals for the harms they have specifically suffered. The second asserts that a society should distribute its goods according to merit. 

\subsection{Overview}

Disentangling legitimate and illegitimate discrimination in insurance is a challenging task for actuaries and data scientists but often required by regulation. A popular example is the 2004 EU Goods and Services Directive, \cite{genderdirective}, that requires ``gender-neutral" insurance premiums, which in effects imposes neutral prices across the sensitive variable. To highlight the core of the problem, we will first explain why predictive models are important in insurance by in Section \ref{sec:2:predictive:insurance} and introduce the ``balance property", which is the mathematical translation of the definition of insurance (``{\em the contribution of the many to the misfortune of the few~}"). In Section \ref{sec:3:distance}, we then present distance measures between distributions, with a focus on the Wasserstein distance, and its connections to matching and the construction of counterfactual observations, as in \cite{charpentier2023transport}. Section \ref{sec:4:wassertein:quantify} illustrates why the Wasserstein distance is an appropriate tool to quantify fairness between the scores of different groups. {In line with previous research conducted in \cite{gouic2020projection} and \cite{chzhen2020fair}, s}ection \ref{sec:5:wassertein:mitigate} then introduces the Wasserstein barycenter to enable the creation of a score distribution ``between" groups that also achieves the balance property we seek. Finally, we will illustrate that technique in Section \ref{sec:6:case:study} on real insurance data\footnote{see \href{https://github.com/Bias2023/FairInsurance}{\footnotesize{\sffamily https://github.com/Bias2023/FairInsurance}}.}.

%\fran{Should we put some related work in Algorithlmic Fairness? especially when we talk about Wasserstein, examples: \cite{gouic2020projection, jiang2020wasserstein, chzhen2020fair}}
%\fran{-> I have added a sentence here, perhaps it's not very elegant ..}

\section{Predictive Models in Insurance}\label{sec:2:predictive:insurance}

The insurance business is characterised by an inverted production cycle. In return for a premium - the amount of which is known when the contract is taken out - the insurer undertakes to cover a risk, the unknown date and amount, according to the definition of ``{actuarial pricing}". In order to do this, the insurer will pool the risks within a mutuality. Insurance's universal secret is therefore the pooling of a large number of insurance contracts within a mutuality, in order to allow compensation to be made between the risks that have been damaged and those for which the insurer has collected premiums without having had to pay out any benefits. To use \citeauthor{chaufton1886assurances}'s \citeyear{chaufton1886assurances} formulation, insurance is the ``{\em compensation of the effects of chance by mutuality organised according to the laws of statistics}". If the use of the expected loss as a premium has been motivated for over a hundred years, it would seem legitimate to use the conditional expected value as a premium principle, for some appropriate risk factors $\boldsymbol{x}$. To formalize this, we first consider the definition of the pure premium:
%On another historical note though, a 1909 law from Kansas allows an insurance commissioner to review rates to ensure that they were not ``{\em excessive, inadequate, or unfairly discriminatory with regards to individuals}"\index{discrimination!unfair}\index{unfairly discriminatory}, as mentioned in \cite{powell2020risk}. Since then, the idea of ``{\em unfairly discriminatory}" insurance rates has been discussed in many States.
\begin{definition}[Pure premium (Heterogeneous risks)]\label{def:pure:premium:heterogeneous} Let $Y$ be the non-negative random variable corresponding to the total annual loss associated with a given policy, associated with covariates $\boldsymbol{X}=\boldsymbol{x}$, the pure premium is the regression function $\mu(\boldsymbol{x})=\mathbb{E}[Y|\boldsymbol{X}=\boldsymbol{x}]$.
\end{definition}
By the law of total expectations it can be written,
$$
\mathbb{E}_{Y}[Y] = \mathbb{E}_{\boldsymbol{X}}\big[ \mathbb{E}_{Y|\boldsymbol{X}}[Y|\boldsymbol{X}] \big] = \mathbb{E}_{\boldsymbol{X}}\big[ \mu(\boldsymbol{X}) \big],
$$
which gives rise to a desirable property we want any trained model $m$ to have
\begin{definition}[Balance Property]\label{def:pure:premium:heterogeneous} A model $m$, used to predict the pure premium $\mu$, satisfies the balance property if $\mathbb{E}_{\boldsymbol{X}}[m(\boldsymbol{X})]=\mathbb{E}_{Y}[Y]$.
\end{definition}
which boils down to having predictions that are correct on average. This definition does not impose limits on the statistical discrimination though. On another historical note, a 1909 law from Kansas allows an insurance commissioner to review rates to ensure that they were not ``{\em excessive, inadequate, or unfairly discriminatory with regards to individuals}"\index{discrimination!unfair}\index{unfairly discriminatory}, as mentioned in \cite{powell2020risk}. Since then, the idea of ``{\em unfairly discriminatory}" insurance rates has been discussed in many States. We illustrate this issue through a simple working example. In the simplest actuarial models, the annual loss $Y$ is related to a single random event, with a fixed cost (which is the case in most life insurance contracts). Therefore, the pure premium is a linear function of the score $\mu(\boldsymbol{x})=\mathbb{P}[Y=1|\boldsymbol{X}=\boldsymbol{x}]$, where $Y$ in a binary variable indicating the occurrence of a risk. To best illustrate the fairness issues, we will be regarding the score function $\mu$, with respect to some binary sensitive attribute $s$ taking values in $\{\textcolor{couleurA}{\code{A}},\textcolor{couleurB}{\code{B}}\}$.
In Figure \ref{fig:distribution:score:1}, we visualize the distribution of the probability to claim a loss, with the distribution of $m(\boldsymbol{x},s=\textcolor{couleurA}{\code{A}})$ and $m(\boldsymbol{x},s=\textcolor{couleurB}{\code{B}})$, respectively with a plain logistic regression on the left, a gradient boosting model in the middle, and a random forest on the right. The dataset is from real personal motor insurance, used in \cite{charpentierCAS} (obtained as the aggregation of \code{freMPL1}, \code{freMPL2}, \code{freMPL3} and \code{freMPL4}, while keeping only observations with \code{exposure} exceeding $0.9$, to have more simple models to illustrate fairness issues). %This data set is used in this form in \cite{charpentier2023springer}.
Across the different estimators, differences in the predictions between the groups as well as differences with respect to the balance property from Definition \ref{def:pure:premium:heterogeneous} are visible. 
\begin{figure}
    \centering
    \includegraphics[width=\textwidth]{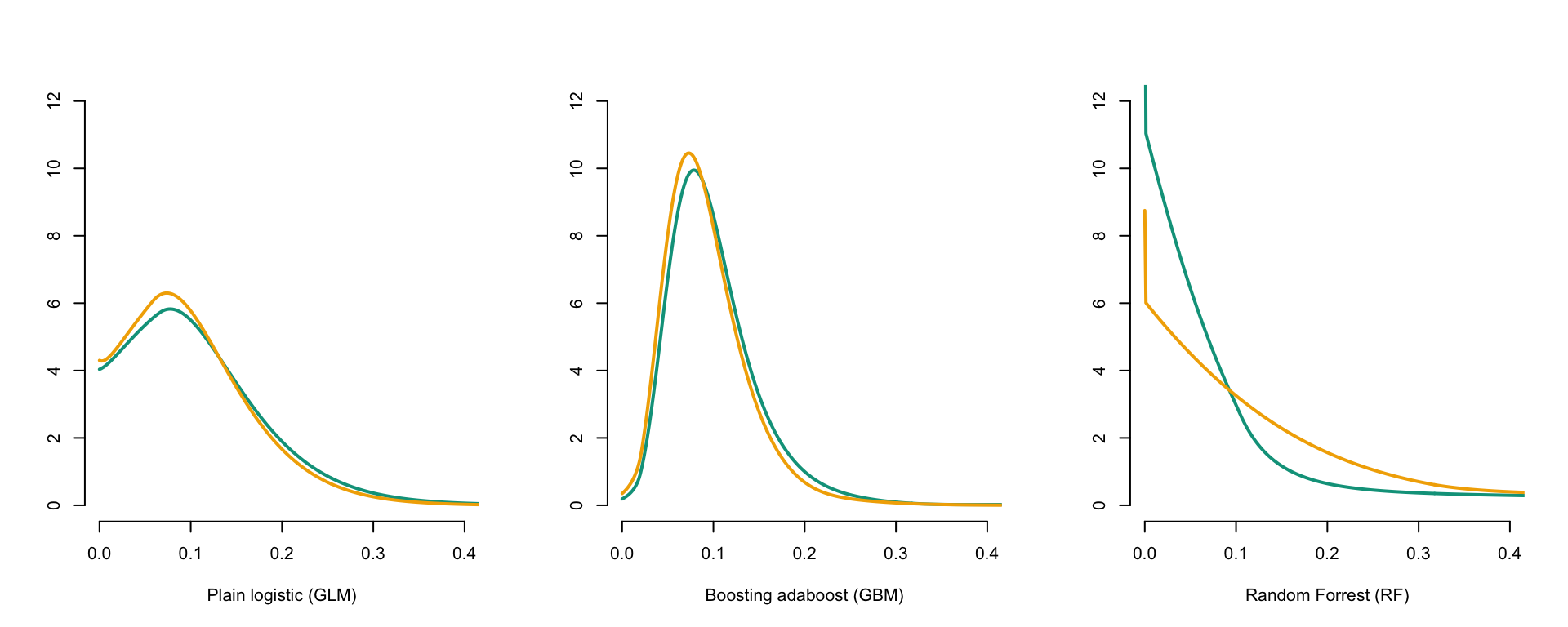}
    \caption{Distributions of $m(\boldsymbol{x},s=\textcolor{couleurA}{\code{A}})$ and $m(\boldsymbol{x},s=\textcolor{couleurB}{\code{B}})$, the probability to claim a loss on a given year, in motor insurance, with three models (GLM, GBM, RF).}
    \label{fig:distribution:score:1}
\end{figure}

\section{Distances Between Distributions}\label{sec:3:distance}

There are several notions to quantify the difference between the group-wise predictions as observed in Figure \ref{fig:distribution:score:1}. For the general case, given two discrete distributions $p$ and $q$, the total variation is the largest possible difference between the probabilities that the two probability distributions can assign to the same event:

\begin{definition}[Total Variation]\cite{camille1881serie,rudin1966real}\label{def:total:variation} For two discrete distributions $p$ and $q$, the total variation distance between $p$ and $q$ is
$$d_{{{\mathrm {TV}}}}(p,q)=\sup_{\mathcal{A}\subset\mathbb{R}}\big\lbrace|p(\mathcal{A})-q(\mathcal{A})|\big\rbrace.
$$
\end{definition}

It should be stressed here that in the context of distributions, \cite{zafar2015fairness} or \cite{zhang2018fairness} suggest to remove the symmetry, to take into account that there is a favored and a disfavored group, and therefore to consider
$$d_{{{\mathrm {TV}}}}(p\|q)=\sup_{\mathcal{A}}\big\lbrace p(\mathcal{A})-q(\mathcal{A})\big\rbrace.
$$
Removing the standard property of symmetry (that we have on distances) yields the concept of "divergence", that is still a non-negative function, positive (in the sense that it is null if and only if "$p=q$", or more precisely ${\displaystyle p\;{\stackrel {a.s.}{=}}\;q}$), and the triangle inequality is not satisfied (even if some satisfy some sort of Pythagorean theorem). As \cite{amari1982differential} explains, it is mainly because divergences are generalizations of "squared distances", not "linear distances".

\begin{definition}[Kullback–Leibler]\cite{kullback1951information}\label{def:KL:divergence} For two discrete distributions $p$ and $q$, Kullback–Leibler divergence of $p$, with respect to $q$ is
$$D_{{{\mathrm {KL}}}}(p\|q)=\sum _{i}p(i)\log {\frac {p(i)}{q(i)}}\!,
$$
and for absolutely continuous distributions,
$$D_{{{\mathrm {KL}}}}(f\|g)=\int _{\mathbb{R}}fpx)\log {\frac {p(x)}{q(x)}}\;\mathrm {d}x\!\text{ or }\int _{\mathbb{R}^d}p(\boldsymbol{x})\log {\frac {p(\boldsymbol{x})}{q(\boldsymbol{x})}}\;d\boldsymbol{x}\!,
$$
in higher dimension.
\end{definition}

Again, this is not a distance (even if it satisfies the nice property ${\displaystyle p\;{\stackrel {a.s.}{=}}\;q}$ if and only if $D_{{{\mathrm {KL}}}}(p\|q)=0$), so we will use the term "divergence" (and notation $D$ instead of $d$).
% Observe that, for two Gaussian distributions, \index{Gaussian}
% $$
% D_{{{\mathrm {KL}}}}(p_1\|p_2)= \frac{1}{2}\left[\log \frac{\sigma_2^2}{\sigma_1^2} + \frac{\sigma_1^2}{\sigma_2^2}+\frac{(\mu_1 - \mu_2)^2}{\sigma_2^2} - 1 \right],
% $$
% and in higher dimension (say $k$),
% $$
% D_{{{\mathrm {KL}}}}(p_1\|p_2)= \frac{1}{2}\left[\log\frac{|\boldsymbol\Sigma_2|}{|\boldsymbol\Sigma_1|} + \text{tr} \{ \boldsymbol\Sigma_2^{-1}\boldsymbol\Sigma_1 \} + (\boldsymbol\mu_2 - \boldsymbol\mu_1)^\top \boldsymbol\Sigma_2^{-1}(\boldsymbol\mu_2 - \boldsymbol\mu_1)-k\right].
% $$
% As proved in \cite{tsybakov2009introduction}, it is possible to find relationships between those measures, such as 
% $$
% d_{{{\mathrm {TV}}}}(p,q)\leq \sqrt{1-\exp[D_{{{\mathrm {KL}}}}(p\|q)]}\text{ or }d_{{{\mathrm {H}}}}(p,q)^2\leq d_{{{\mathrm {TV}}}}(p,q)\leq \sqrt{2}d_{{{\mathrm {H}}}}(p,q).
% $$
It is possible to derive a symmetric divergence measure by averaging with the so-called "dual divergence", or to consider the following approach, with "Jensen-Shannon divergence", 
\begin{definition}[Jensen-Shannon]\cite{lin1991divergence}.\label{def:JS:divergence} The Jensen-Shannon distance is a symmetric distance induced by Kullback-Liebler divergence,
$$\displaystyle{D_{{{\mathrm {JS}}}}(p_1,p_2)={\frac {1}{2}}D_{{{\mathrm {KL}}}}(p_1\|q)+{\frac {1}{2}}D_{{{\mathrm {KL}}}}(p_1\|q),}$$
where $\displaystyle{q={\frac {1}{2}}(p_1+p_2)}$. 
\end{definition}

Another popular distance is the Wasserstein distance, also called Mallows' distance, from \cite{mallows1972note},

\begin{definition}[Wasserstein]\label{def:W:divergence}\cite{vaserstein1969markov}.
Consider two measures on $p$ and $q$ on $\mathbb{R}^d$, with a norm $\|\cdot\|$ (on $\mathbb{R}^d$). Then define
$$\displaystyle W_{k}(p ,q )=\left(\inf _{\pi \in \Pi (p,q )}\int _{{\mathbb{R}^d}\times \mathbb{R}^d}\|\boldsymbol{x}-\boldsymbol{y}\|^{k}\mathrm {d} \pi (\boldsymbol{x},\boldsymbol{y})\right)^{1/k},$$
where ${\displaystyle \Pi (p ,q )}$ is the set of all couplings of $p$ and $q$. 
\end{definition}
Throughout this article, unless stated otherwise, we will consider the Wasserstein distance to be the $W_2$ and $d$ the Euclidean distance. As mentioned in \cite{villani2009optimal}, the total variation distance arises quite naturally as the optimal transportation cost, when the cost function is, $\ell_{0/1}$, or $\boldsymbol{1}(x\neq y)$, since
$$
d_{{{\mathrm {TV}}}}(p,q)=\inf_{\pi \in \Pi (p,q )}\big\lbrace\mathbb{P}[X\neq Y],~(X,Y)\sim \pi\big\rbrace=\inf_{\pi \in \Pi (p,q )}\big\lbrace\mathbb{E}[\ell_{0/1}(X, Y)],~(X,Y)\sim \pi\big\rbrace.
$$
With Wasserstein-distance, we consider
$$
\inf_{\pi \in \Pi (p,q )}\big\lbrace\mathbb{E}[\ell(X, Y)],~(X,Y)\sim \pi\big\rbrace\text{ or }
\inf_{\pi \in \Pi (p,q )}\left\lbrace\int \ell(x,y)\pi(dx,dy)\right\rbrace.
$$
The connection with ``transport" is obtained as follows: given $\mathcal{T}:\mathbb{R}^k\rightarrow\mathbb{R}^k$, define the ``{\em push-forward}'' measure, 
$$
\mathbb{P}_1(A)=  \mathcal{T}_{\#}\mathbb{P}_0(A)= \mathbb{P}_0\big(\mathcal{T}^{-1}(A)\big),~\forall A\subset \mathbb{R}^k.
$$
An optimal transport $\mathcal{T}^\star$ (in Brenier's sense,  from \cite{brenier1991polar}, see \cite{villani2009optimal} or \cite{galichon2016optimal}) from $\mathbb{P}_0$ towards $\mathbb{P}_1$ will be solution of 
$$
\mathcal{T}^\star\in \underset{\mathcal{T}:\mathcal{T}_{\#}\mathbb{P}_0=\mathbb{P}_1}{\text{arginf}}\left\lbrace\int_{\mathbb{R}^k} \ell(\boldsymbol{x},\mathcal{T}(\boldsymbol{x}))d\mathbb{P}_0(\boldsymbol{x})\right\rbrace,
$$

In dimension 1 (distributions on $\mathbb{R}$), let $F_0$ and $F_1$ denote the cumulative distribution function, and $F_0^{-1}$ and $F_1^{-1}$ denote quantiles\index{quantile}. Then
$${\displaystyle W_{k}(p_0,p_1)=\left(\int _{0}^{1}\left|F_{0}^{-1}(u)-F_{1}^{-1}(u)\right|^{k}\,\mathrm {d} u\right)^{1/k}},$$
and one can prove that the optimal transport $\mathcal{T}^\star$ is a monotone transformation. More precisely,
$$
\mathcal{T}^\star:x_0\mapsto x_1 = F_1^{-1}\circ F_0(x_0).
$$
For empirical measures, in dimension 1, the distance is a simple function of the order statistics:
$${\displaystyle W_{k}(\boldsymbol{x},\boldsymbol{y})=\left({\frac {1}{n}}\sum _{i=1}^{n}|x_{(i)}-y_{(i)}|^{k}\right)^{1/k}}.$$

Observe that, for two Gaussian distributions, and the Euclidean distance,
$${\displaystyle W_{2}(p _{0},p _{1})^{2}=(\mu_{1}-\mu_{0})^{2}+ {\bigl(\sigma_1-\sigma_0\bigr)^2},}$$
and in higher dimension,
$${\displaystyle W_{2}(p _{0},p _{1})^{2}=\|\boldsymbol\mu_{1}-\boldsymbol\mu_{0}\|_{2}^{2}+\operatorname {tr} {\bigl (}\boldsymbol\Sigma_{0}+\boldsymbol\Sigma_{1}-2{\bigl (}\boldsymbol\Sigma_{1}^{1/2}\boldsymbol\Sigma_{0}\boldsymbol\Sigma_{1}^{1/2}{\bigr )}^{1/2}{\bigr )}.}$$

If variances are equal, we can write simply
$$
\begin{cases}
W_{2} (p_0, p_1)^2 = \| \boldsymbol\mu_1 - \boldsymbol\mu_0 \|_2^2 =(\boldsymbol\mu_1 - \boldsymbol\mu_0)^\top (\boldsymbol\mu_1 - \boldsymbol\mu_0)\\
D_\text{KL} (p_0\|p_1) = (\boldsymbol\mu_1 - \boldsymbol\mu_0)^\top \boldsymbol\Sigma_1^{-1}(\boldsymbol\mu_1 - \boldsymbol\mu_0)
\end{cases}
$$
And in that Gaussian case, there is an explicit expression for the optimal transport, which is simply an affine map (see \cite{villani2003optimal} for more details). In the univariate case, $x_1 = \mathcal{T}^\star_{\mathcal{N}}(x_0) = \mu_1+ \displaystyle{\frac{\sigma_1}{\sigma_0}(x_0-\mu_0)}$, while in the multivariate case, an analogous expression can be derived:$$
\boldsymbol{x}_1 = \mathcal{T}^\star_{\mathcal{N}}(\boldsymbol{x}_0)=\boldsymbol{\mu}_1 + \boldsymbol{A}(\boldsymbol{x}_0-\boldsymbol{\mu}_0),
$$
where $\boldsymbol{A}$ is a symmetric positive matrix that satisfies $\boldsymbol{A}\boldsymbol{\Sigma}_0\boldsymbol{A}=\boldsymbol{\Sigma}_1$, which has a unique solution given by $\boldsymbol{A}=\boldsymbol{\Sigma}_0^{-1/2}\big(\boldsymbol{\Sigma}_0^{1/2}\boldsymbol{\Sigma}_1\boldsymbol{\Sigma}_0^{1/2}\big)^{1/2}\boldsymbol{\Sigma}_0^{-1/2}$, where $\boldsymbol{M}^{1/2}$ is the square root of the square (symmetric) positive matrix $\boldsymbol{M}$ based on the Schur decomposition ($\boldsymbol{M}^{1/2}$ is a positive symmetric matrix), as described in \cite{higham2008functions}.

\section{Wasserstein Distance to Quantify Discrimination}\label{sec:4:wassertein:quantify}
The definition of fairness is somewhat more complicated than simply using a distance metric. As pointed out by \cite{caton2020fairness}, there are at least a dozen ways to define (formally) the fairness of a classifier, or more generally of a model. For example, one can wish for independence between the score and the group membership, $m(\boldsymbol{Z})\indep S$, or between the prediction (as a class) and the protected variable $\widehat{Y}\indep S$.  

\begin{definition}[Independence]\cite{barocas2017fairness} A model $m$ satisfies the independence property if $m(\boldsymbol{X},S)\indep S$, with respect to the distribution $\mathbb{P}$ of the triplet $(\boldsymbol{X},S,Y)$.
\end{definition}\label{def:independence}

From this property, we can define the concept of ``{\em demographic parity}" (also called ``{\em statistical fairness}", ``{\em equal parity}", ``{\em equal acceptance rate}" or simply ``{\em independence}", as mentioned in \cite{calders2010three}).

\begin{definition}[Weak Demographic Parity] A model $m$ satisfies weak demographic parity if 
$$
\mathbb{E}[m(\boldsymbol{X},S)|S=\textcolor{couleurA}{\code{A}}] = 
\mathbb{E}[m(\boldsymbol{X},S)|S=\textcolor{couleurB}{\code{B}}]\text{ or } 
\mathbb{E}_{\mathbb{P}_{\textcolor{couleurA}{\code{A}}}}[m(\boldsymbol{X},S)]=
\mathbb{E}_{\mathbb{P}_{\textcolor{couleurB}{\code{B}}}}[m(\boldsymbol{X},S)].
$$
\end{definition}

A stronger condition can be obtained if we ask to have equality of the distributions of scores, instead of the average value. A classical definition is based on the Total Distance (as in Definition \ref{def:total:variation}),

\begin{definition}[Strong Demographic Parity] A decision function $\widehat{y}$ satisfies strong demographic parity if $\widehat{Y}\indep S$, i.e. for all $\mathcal{A}\subset\mathbb{R}$,
$$
\mathbb{P}[\widehat{Y}\in\mathcal{A}|S=\textcolor{couleurA}{\code{A}}] = 
\mathbb{P}[\widehat{Y}\in\mathcal{A}|S=\textcolor{couleurB}{\code{B}}],~\forall \mathcal{A}\subset\mathcal{Y}\text{ or }d_{{{\mathrm {TV}}}}(\mathbb{P}_{\textcolor{couleurA}{\code{A}}},\mathbb{P}_{\textcolor{couleurB}{\code{B}}})=0,
$$
where $\mathbb{P}_{\textcolor{couleurA}{\code{A}}}$ and $\mathbb{P}_{\textcolor{couleurB}{\code{B}}}$ denote the conditional distributions of the score $m(\boldsymbol{X},S)$.
\end{definition}
This notion naturally extends to the Wasserstein distance as

\begin{proposition}A model $m$ satisfies the strong demographic parity property if and only if $W_2(\mathbb{P}_{\textcolor{couleurA}{\code{A}}},\mathbb{P}_{\textcolor{couleurB}{\code{B}}})=0.
$
\end{proposition}

It is also particularly easy to visualize that property on Figure \ref{fig:match:gender}, with on the $x$-axis, the distribution of the score in group \textcolor{couleurA}{\code{A}}, and on the $y$-axis the distribution of the score in group \textcolor{couleurB}{\code{B}}. The \textcolor{couleurzero}{plain line} is the (monontonic) optimal transport $\mathcal{T}^\star$. If that line is on the diagonal, $m$ is fair (for the ``strong demographic parity" criteria).

\begin{figure}
    \centering
    \includegraphics[width=.32\textwidth]{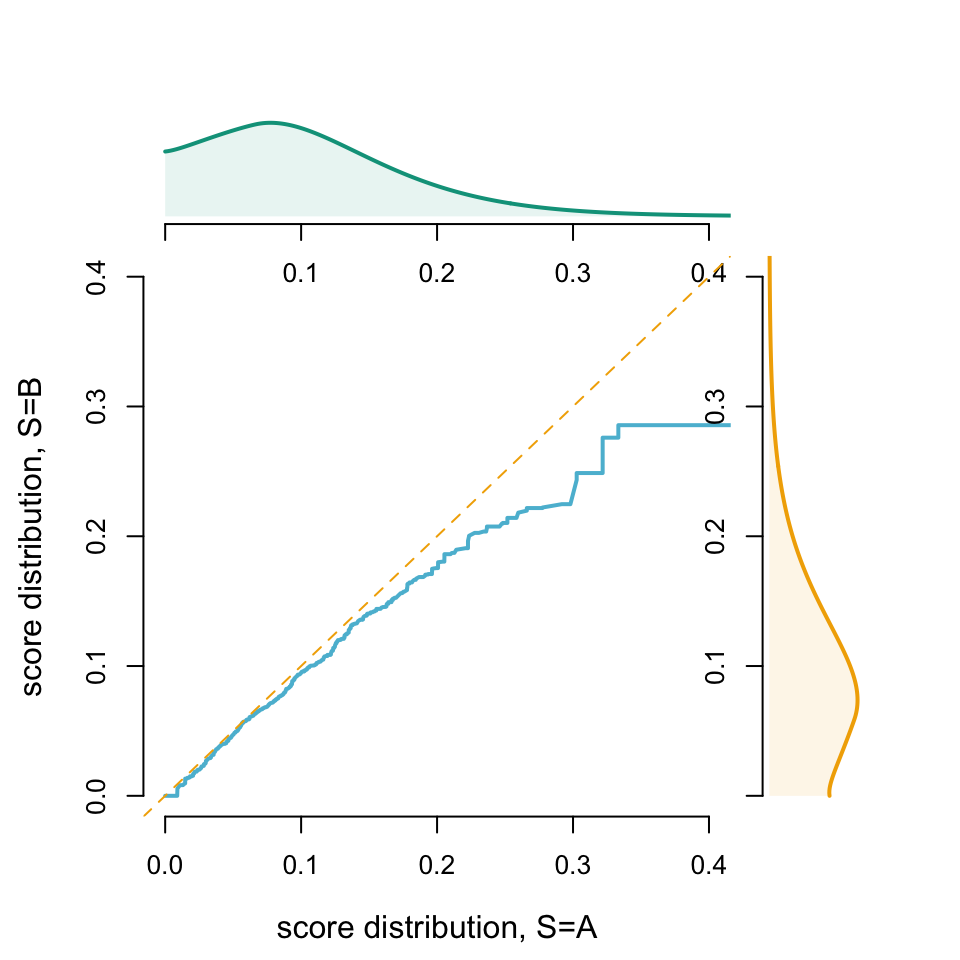}\includegraphics[width=.32\textwidth]{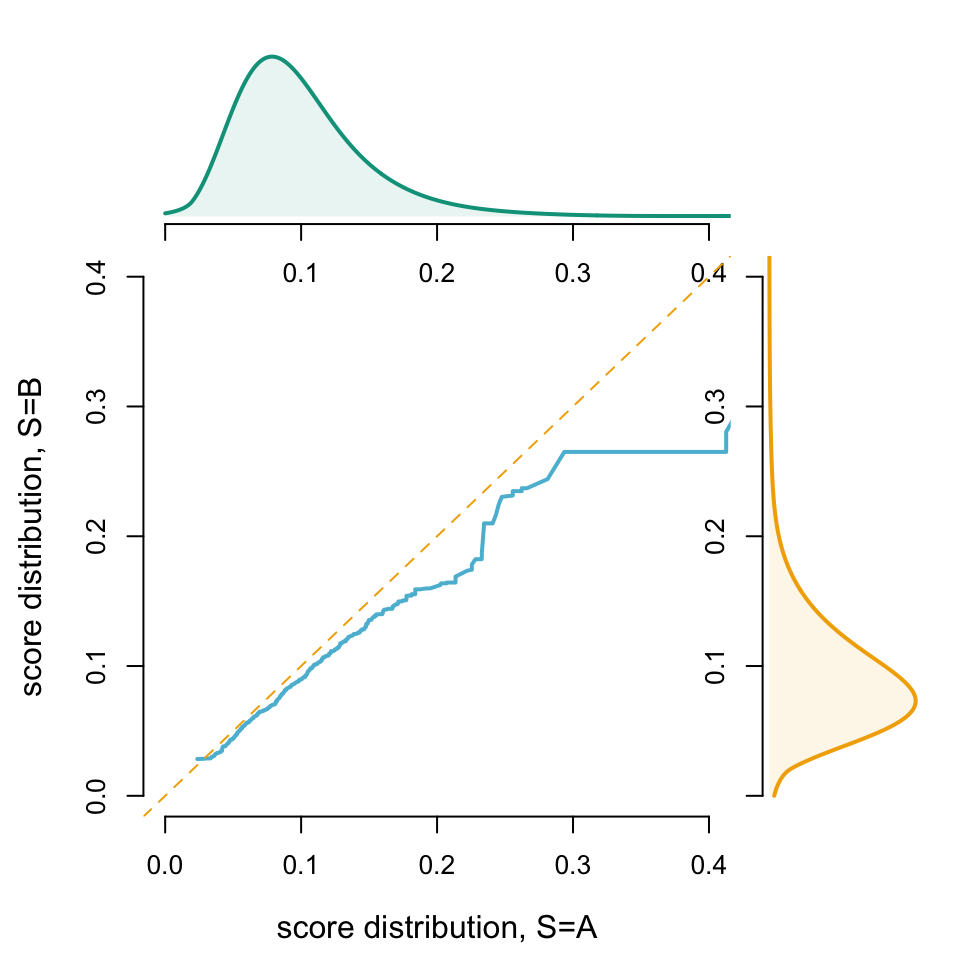}\includegraphics[width=.32\textwidth]{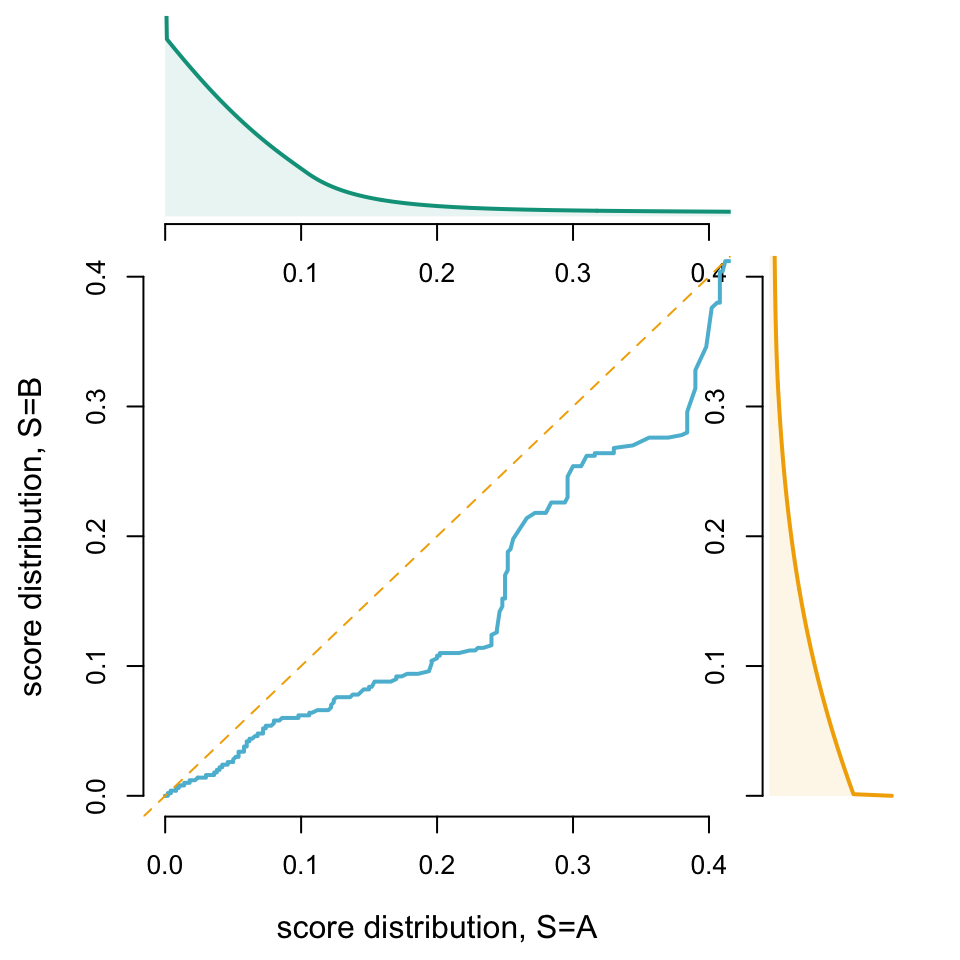}
    \caption{Matching between $m(\boldsymbol{x},s=\textcolor{couleurA}{\code{A}})$ and $m(\boldsymbol{x},s=\textcolor{couleurB}{\code{B}})$, where $m$ is (from the left to the right) GLM, GBM and RF.}
    \label{fig:match:gender}
\end{figure}

\section{Wasserstein Barycenters to Mitigate Discrimination}\label{sec:5:wassertein:mitigate}
Mitigating discrimination can be achieved through several techniques. For example, a simple approach if weak demographic parity is not satisfied, in the sense that $\mathbb{E}_{\mathbb{P}_{\textcolor{couleurA}{\code{A}}}}[m(\boldsymbol{X})]\neq
\mathbb{E}_{\mathbb{P}_{\textcolor{couleurB}{\code{B}}}}[m(\boldsymbol{X})]$ would be to consider
$$
m^\star(\boldsymbol{x},s) = \frac{\mathbb{E}_{\mathbb{P}}[m(\boldsymbol{X},S)]}{\mathbb{E}_{\mathbb{P}_{{{s}}}}[m(\boldsymbol{X},s)]}\cdot m(\boldsymbol{x},s)\text{ for a policyholder in group }s.
$$
As a numerical example from our dataset, overall, a single policyholder has $8.67\%$ chance to claim a loss, $8.94\%$ for a man (group \code{A}) and $8.20\%$ for a woman (group \code{B}). Because of this difference, in order to get a fair model, ``gender-neutral", the premium for a woman should be $8.67/8.20=1.058$ (or $5.8\%$) higher, $m^\star(\boldsymbol{x},s)=1.058\cdot m(\boldsymbol{x},s)$, and $3\%$ lower than the predicted one, for men. This approach is perhaps a bit too simplistic, as it ignores differences between the group distributions. An alternative is to consider the use of a barycenter of distributions, as done for example in  \cite{gouic2020projection, jiang2020wasserstein, chzhen2020fair,charpentier2023ecml}%\cite{charpentier2023ecml, charpentier2023neurips}
. Recall that barycenters of $\{\boldsymbol{z}_1,\cdots,\boldsymbol{z}_n\}$, in standard Euclidean spaces, are simply "{\em weighted averages}'', defined as  solution of $$
\boldsymbol{z}^\star=\underset{\boldsymbol{z}}{\text{argmin}}\left\lbrace\sum_{i=1}^n \omega_i d(\boldsymbol{z},\boldsymbol{z}_i)^2 \right\rbrace,
$$
for some weights $\omega_i\geq 0$, and where $d$ is the standard Euclidean distance.  This can be extended to more general spaces, such as measures. We can therefore define some sort of average measure, solution of
$$
\mathbb{P}^\star=\underset{\mathbb{Q}}{\text{argmin}}\left\lbrace\sum_{i=1}^n \omega_i d(\mathbb{Q},\mathbb{P}_i)^2 \right\rbrace,
$$
for some distance (or divergence) $d$, as in \cite{nielsen2011burbea}. Those are also called "centroids" associated with measures $\mathcal{P}=\{\mathbb{P}_1,\cdots,\mathbb{P}_n\}$, and weights $\boldsymbol{\omega}$. For instance, \cite{jeffreys1946invariant} consider the empirical case of "{\em averaging histograms}" (and not theoretical measures $\mathbb{P}_i$), extended in \cite{nielsen2009sided} as the \cite{nielsen2013jeffreys} as "{\em generalized Kullback–Leibler centroid}" (see Definition \ref{def:KL:divergence} for Kullback–Leibler divergence, and the symmetric extension in Definition \ref{def:JS:divergence} base on some "average" measure). 

An alternative (see \cite{agueh2011barycenters} and Definition \ref{def:W:divergence}) is to use the Wasserstein distance $W_2$. As shown in \cite{santambrogio2015optimal}, if one of the measures $\mathbb{P}_i$ is absolutely continuous, the minimization problem has a unique solution. As discussed in Section 5.5.5 in \cite{santambrogio2015optimal}, it is possible to simple a simple version for univariate measures. Given a reference measure, say  $\mathbb{P}_1$, it is possible to write the barycenter as the "{\em average push-forward}" transformation of $\mathbb{P}_1$: if $\mathbb{P}_i = \mathcal{T}^{1\to i}_{\#}\mathbb{P}_1$ (with the convention that $\mathcal{T}^{1\to 1}_{\#}$ is the identity), 
$$
\mathbb{P}^\star = \left(\sum_{i=1}^n \omega_i\mathcal{T}^{1\to i}\right)_{\#}\mathbb{P}_1.
$$
And in the univariate case, $\mathcal{T}^{1\to i}$ is simply a rearrangement, defined as $\mathcal{T}^{1\to i}=F_i ^{-1}\circ F_1$, where $F_i(t)=\mathbb{P}_i((-\infty,t])$ and $F_i^{-1}$ is its generalized inverse. Note that Wasserstein Barycenter is also named "{\em  Fr\'echet mean of distributions}" in \cite{petersen2019frechet}. As discussed in \cite{alvarez2018wide}, moments and risk measures associated with $\mathbb{P}^\star$ can be expressed simply from associated measures on $\mathbb{P}_i$'s and $\boldsymbol{\omega}$. 

\begin{definition}[Fair barycenter score] Given two scores $m(\boldsymbol{x},s=\textcolor{couleurA}{\code{A}})$ and $m(\boldsymbol{x},s=\textcolor{couleurB}{\code{B}})$, the ``fair barycenter score" is
$$
\begin{cases}
  m^\star(\boldsymbol{x},s=\textcolor{couleurA}{\code{A}}) = \mathbb{P}[S=\textcolor{couleurA}{\code{A}}] \cdot m(\boldsymbol{x},s=\textcolor{couleurA}{\code{A}}) +\mathbb{P}[S=\textcolor{couleurB}{\code{B}}] \cdot F_{\textcolor{couleurB}{\code{B}}}^{-1} \circ  F_{\textcolor{couleurA}{\code{A}}}\big(m(\boldsymbol{x},s=\textcolor{couleurA}{\code{A}})\big) \\
   m^\star(\boldsymbol{x},s=\textcolor{couleurB}{\code{B}}) = \mathbb{P}[S=\textcolor{couleurA}{\code{A}}]\cdot F_{\textcolor{couleurA}{\code{A}}}^{-1} \circ  F_{\textcolor{couleurB}{\code{B}}}\big(m(\boldsymbol{x},s=\textcolor{couleurB}{\code{B}})\big) +\mathbb{P}[S=\textcolor{couleurB}{\code{B}}]  \cdot m(\boldsymbol{x},s=\textcolor{couleurB}{\code{B}}).
\end{cases}
$$
\end{definition}

\begin{proposition}
The score $m^\star$ is balanced.
\end{proposition}
\begin{proof}
Trivial from the law of total expectation, and since weights are $\omega_i=\mathbb{P}[S=i]$,
\end{proof}

% file:///Users/arthurcharpentier/Downloads/17-BEJ957.pdf moyenne

In the case of Gaussian distributions (as in \cite{mallasto2017learning}) $\mathcal{N}(\boldsymbol{\mu}_i,\boldsymbol{\Sigma}_i)$,
% Jeffrey-Kullback–Leibler-centroid of those distribution would be
% $$\mathcal{N}(\boldsymbol{\mu}^*,\boldsymbol{\Sigma}^*),\text{ where }
% \boldsymbol{\mu}^* = \sum_{i=1}^n \omega_i \boldsymbol{\mu}_i\text{ and }\boldsymbol{\Sigma}^*=\sum_{i=1}^n \omega_i \boldsymbol{\Sigma}_i,
% $$
Wasserstein barycenter is here
$$\mathcal{N}(\boldsymbol{\mu}^\star,\boldsymbol{\Sigma}^\star),\text{ where }
\boldsymbol{\mu}^\star = \sum_{i=1}^n \omega_i \boldsymbol{\mu}_i,
$$
and where $\boldsymbol{\Sigma}^\star$ is the unique positive definite matrix such that
$$
\boldsymbol{\Sigma}^\star= \sum_{i=1}^n\omega_i \big(\boldsymbol{\Sigma}^{\star1/2}\boldsymbol{\Sigma}_i\boldsymbol{\Sigma}^{\star1/2}\big)^{1/2}.
$$
% In the univariate case, with two Gaussian measures, the difference is that in the first case, the variance is the average of variances, while in the second case, the standard deviation is the average of standard deviations,
% $$
% \begin{cases}
% \sigma^* = \sqrt{\omega_1\sigma_1^2+\omega_2\sigma_2^2}:~ \text{Jeffrey-Kullback–Leibler centroid}\\
% \sigma^\star = \omega_1\sigma_1+\omega_2\sigma_2:~\text{Wassertein barycenter}.\\
% \end{cases}
% $$
In Figure \ref{fig:distribution:bary:matching}, inspired from Figure \ref{fig:match:gender}, we can visualize the matching between $m(\boldsymbol{x},s=\textcolor{couleurA}{\code{A}})$ and $m^\star(\boldsymbol{x},s=\textcolor{couleurA}{\code{A}})$ on top, and between $m(\boldsymbol{x},s=\textcolor{couleurB}{\code{B}})$ and $m^\star(\boldsymbol{x},s=\textcolor{couleurB}{\code{B}})$ below.
\begin{figure}[!h]
    \centering
    \includegraphics[width=.32\textwidth]{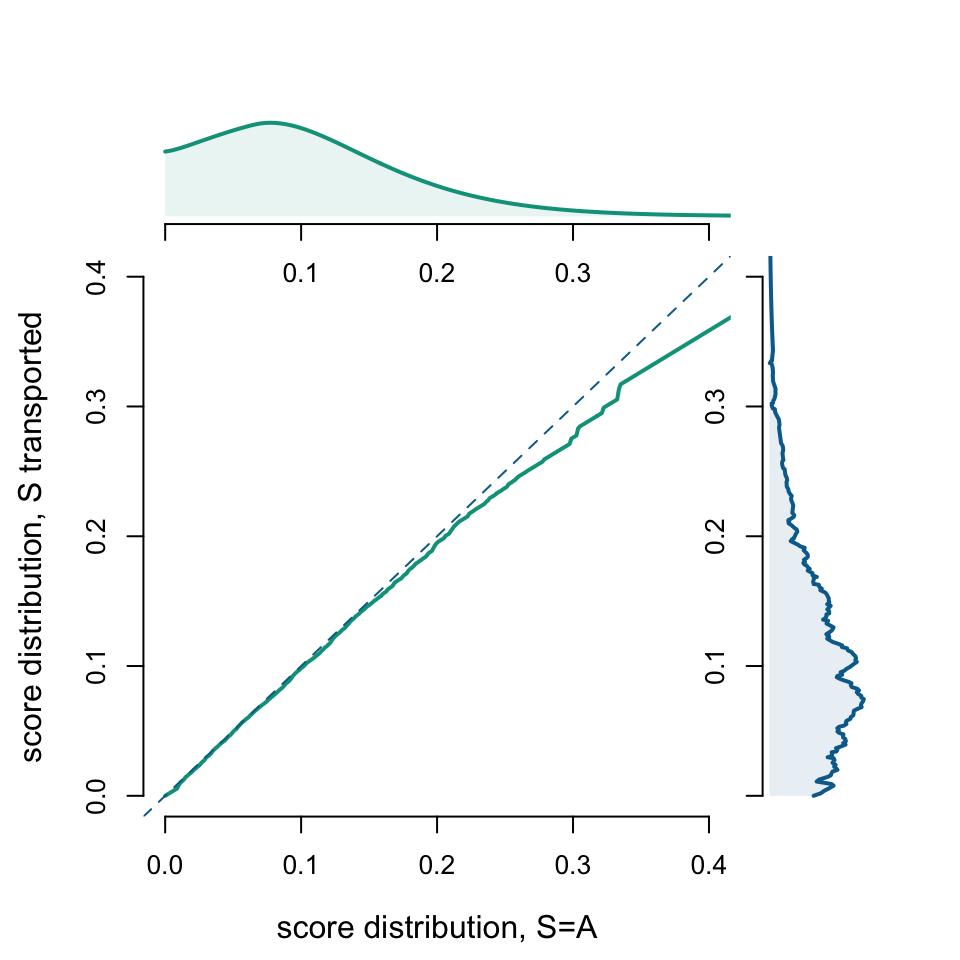} \includegraphics[width=.32\textwidth]{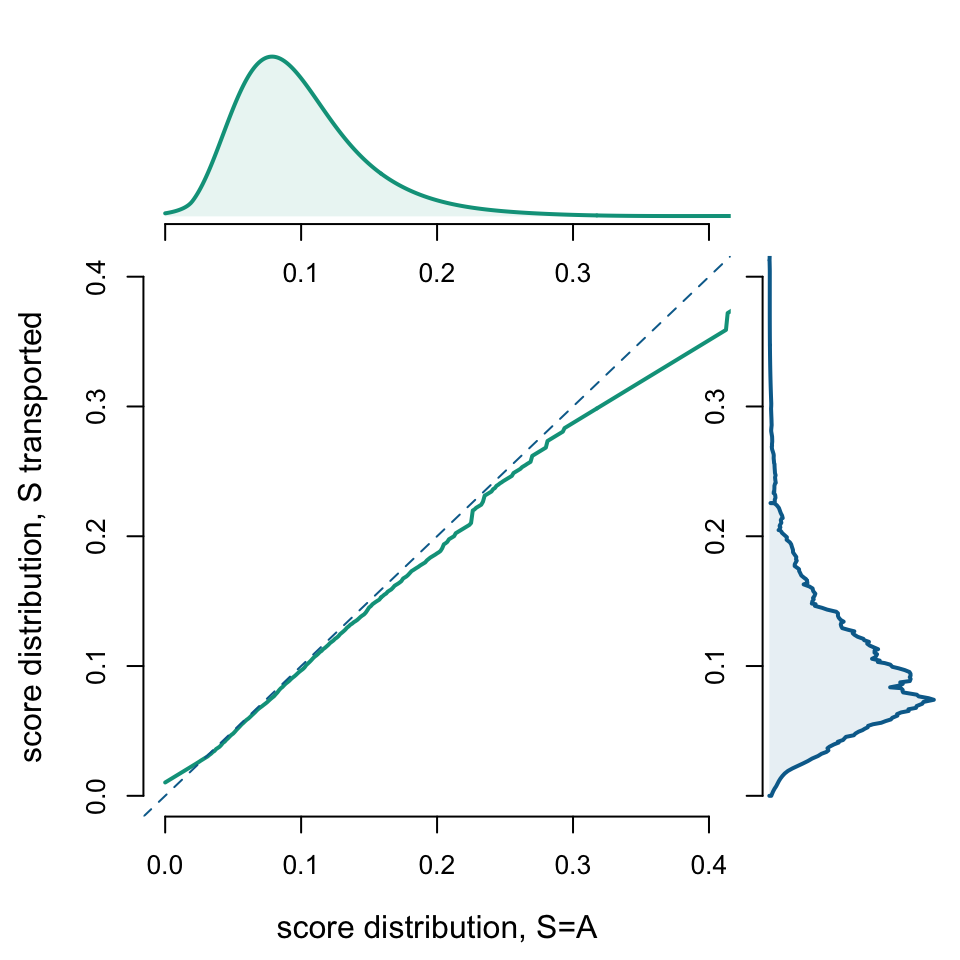} \includegraphics[width=.32\textwidth]{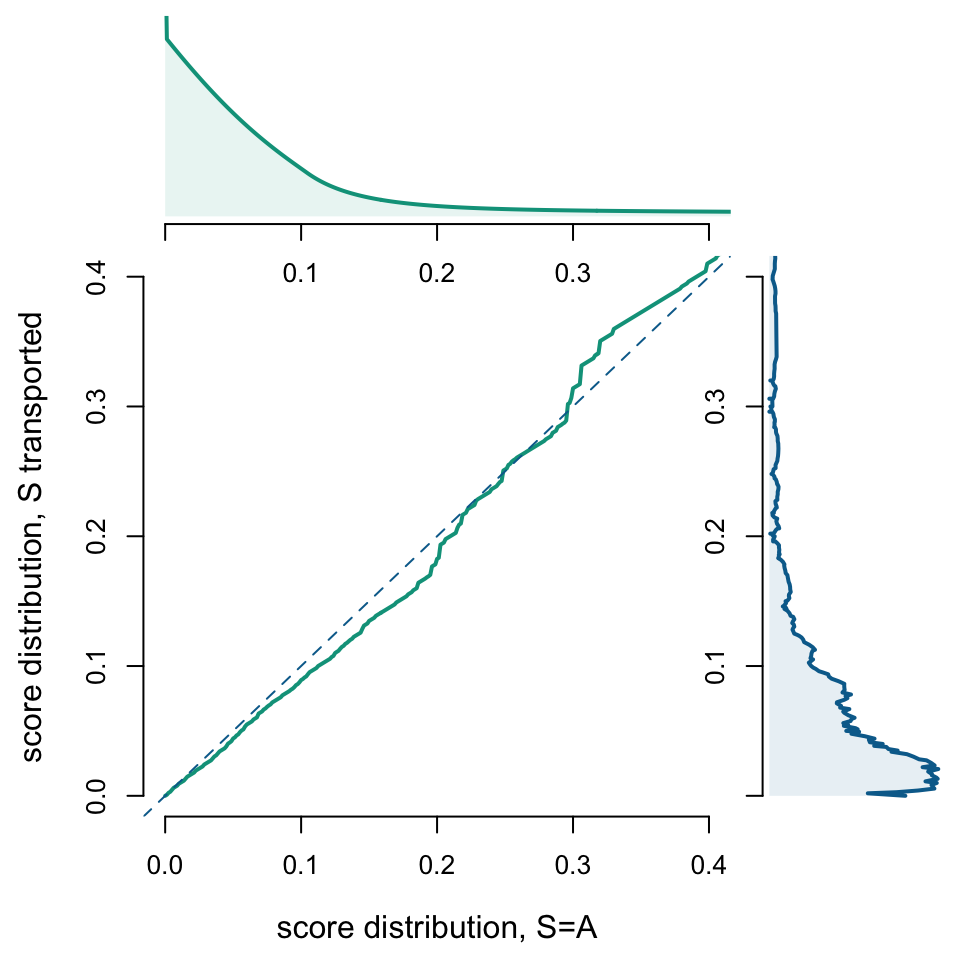}

    \includegraphics[width=.32\textwidth]{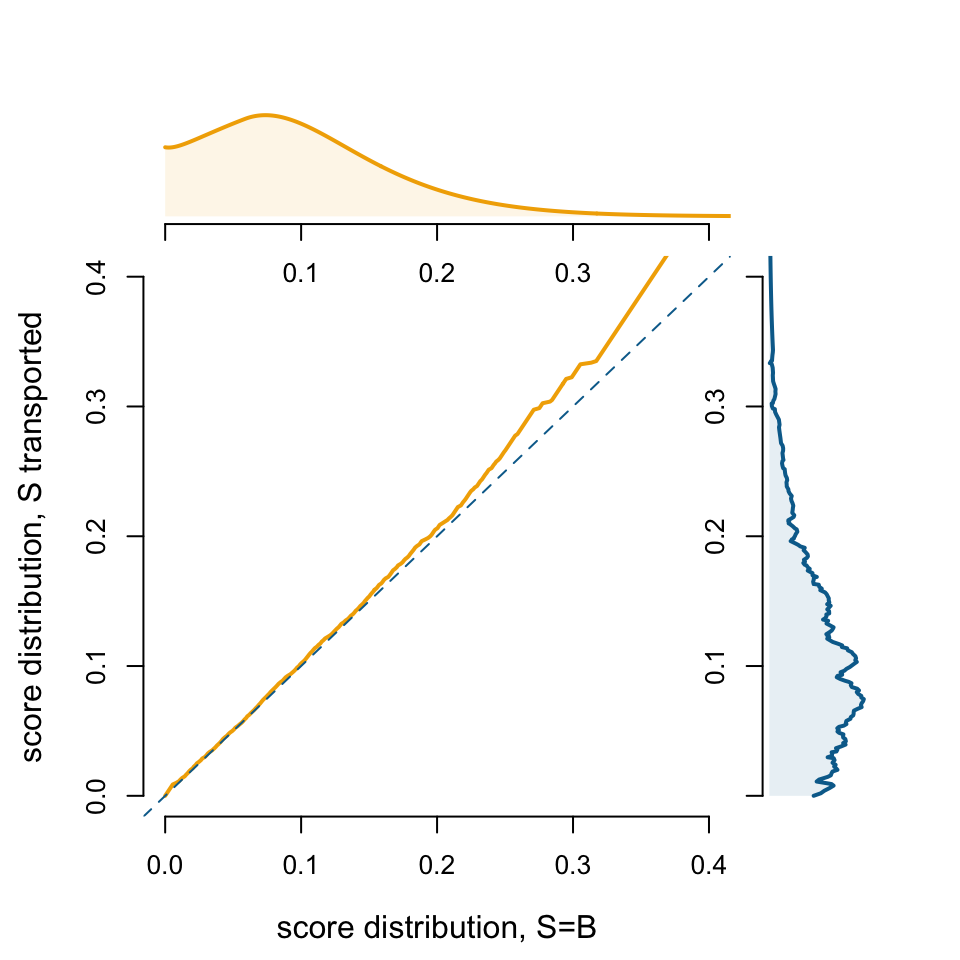} \includegraphics[width=.32\textwidth]{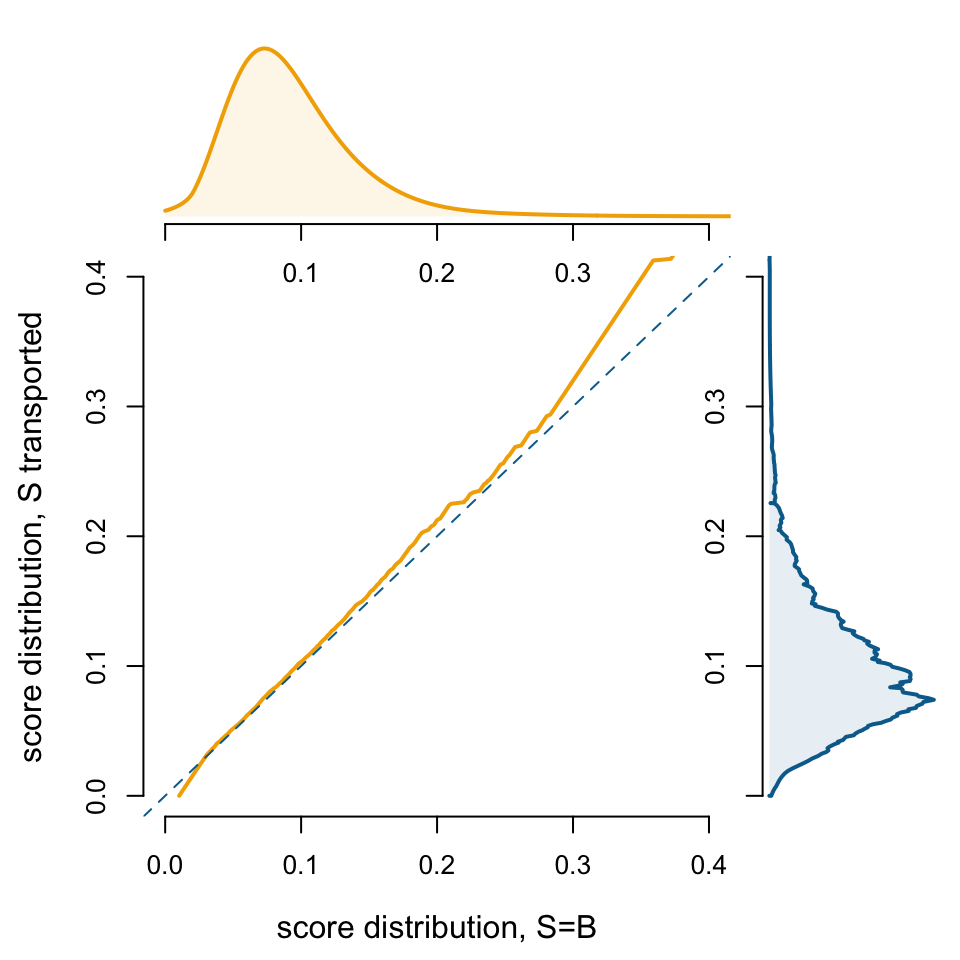} \includegraphics[width=.32\textwidth]{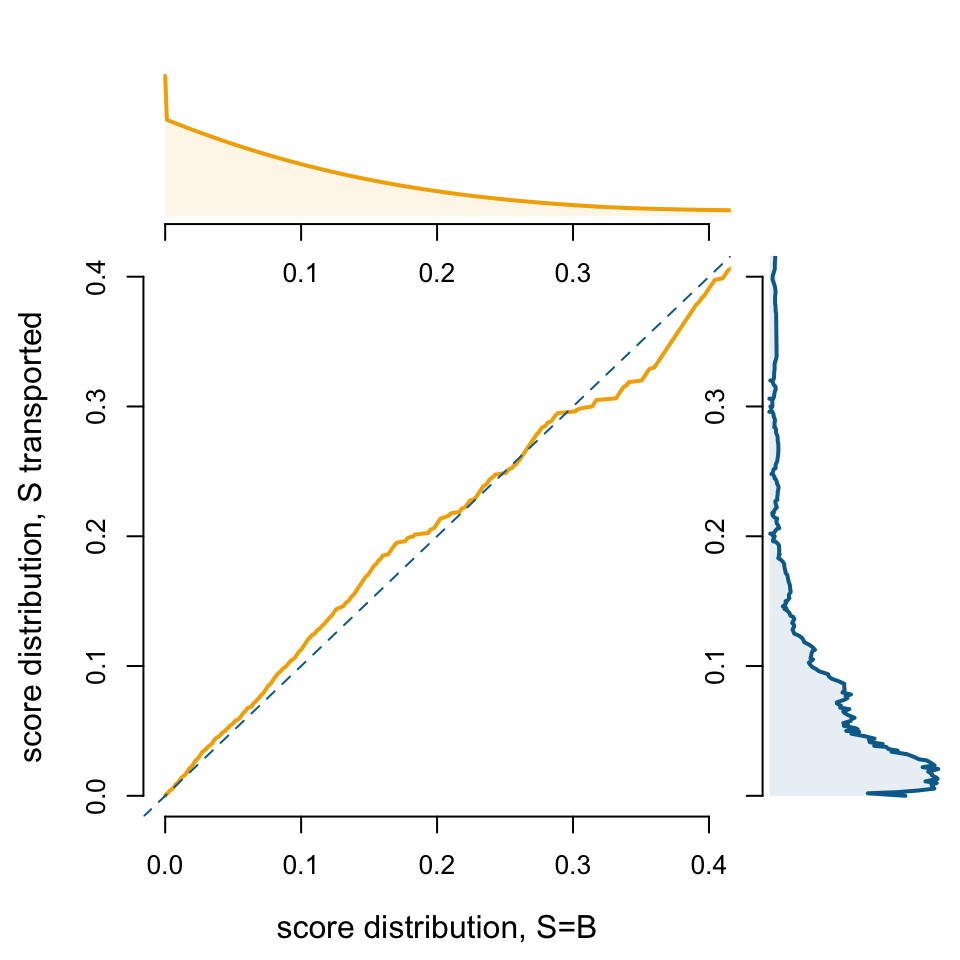}
    \caption{Matching between $m(\boldsymbol{x},s=\textcolor{couleurA}{\code{A}})$ and $m^\star(\boldsymbol{x},s=\textcolor{couleurA}{\code{A}})$, on top, and between $m(\boldsymbol{x},s=\textcolor{couleurB}{\code{B}})$ and $m^\star(\boldsymbol{x},s=\textcolor{couleurB}{\code{B}})$, below, on the probability to claim a loss in motor insurance when $s$ is the gender of the driver.}
    \label{fig:distribution:bary:matching}
\end{figure}
In Figure \ref{fig:distribution:score:scatter}, we have the scatterplot of points $(m(\boldsymbol{x}_i,s_i=\textcolor{couleurA}{\code{A}}),m(^\star\boldsymbol{x}_i))$ and $(m(\boldsymbol{x}_i,s_i=\textcolor{couleurB}{\code{B}}),m(^\star\boldsymbol{x}_i))$.
\begin{figure}[!h]
    \centering
    \includegraphics[width=\textwidth]{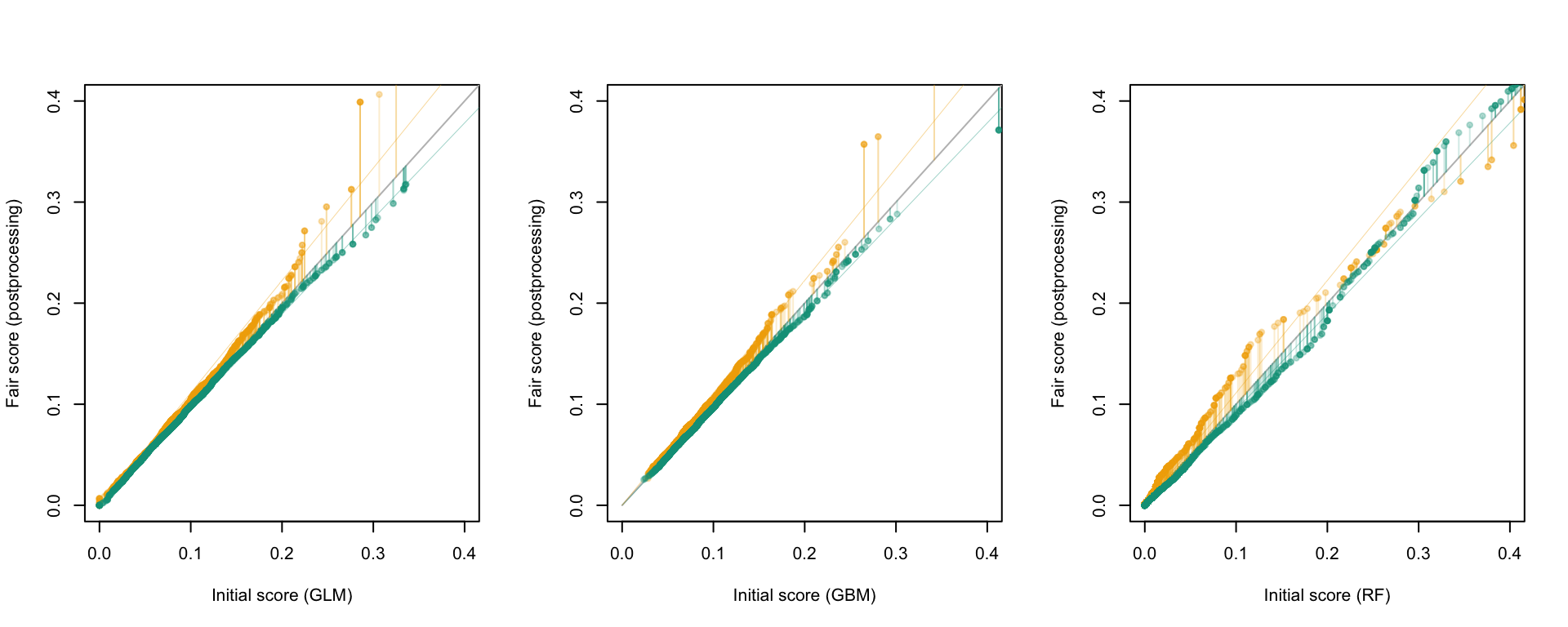}
    \caption{Scatterplot of points $(m(\boldsymbol{x}_i,s_i=\textcolor{couleurA}{\code{A}}),m(^\star\boldsymbol{x}_i),s=\textcolor{couleurA}{\code{A}})$ and $(m(\boldsymbol{x}_i,s_i=\textcolor{couleurB}{\code{B}}),m(^\star\boldsymbol{x}_i),s=\textcolor{couleurB}{\code{B}})$, with three models (GLM, GBM, RF), on the probability to claim a loss in motor insurance when $s$ is the gender of the driver.}
    \label{fig:distribution:score:scatter}
\end{figure}

\section{Motor Insurance Case Study}\label{sec:6:case:study}

{Building upon the preceding study, we emphasize that the dataset employed in this section continues to originate from the \texttt{freMPL} data discussed earlier.}
%\fran{I think we can "remind" the reader that the dataset we use is still freMPL.}

\subsection{Gender of the Main Driver}

Gender is considered as a sensitive attribute in many places around the world. And as strongly stated in \cite{kearns2019ethical}, ``{\em machine learning} (or any predictive model) {\em won't give you anything like gender neutrality `for free' that you didn't explicitly ask for}." Legal obligations often also require neutrality with respect to gender. For example, the 2004 EU Goods and Services Directive, \cite{genderdirective}, aimed to reduce gender gaps in access to all goods and services, discussed for example by \cite{thiery2006fairness}. 
%\fran{Since the workshop is taking place in Europe and we're focusing on the insurance field, perhaps we could highlight the EU Court of Justice's "unisex" directive (December 2012), which encourages gender-neutral equal premiums?}
%\textcolor{red}{In line with the European Court of Justice's ruling on December 21, 2012, EU insurers are now prohibited from considering gender as a determinant when setting insurance premiums.}
In the United States, according to \cite{zebra}, it is forbidden to use the gender in 6 States (California, Hawaii, Massachusetts, Montana, North Carolina and Pennsylvania) and in Canada, \cite{ibc2021}.

To further follow our example, in the entire dataset, we have 64\% men (7973) and 36\% women (4464) registered as ``main driver". Overall, if we consider ``weak demographic parity", $8.2\%$ women claim a loss, against $8.9\%$ women. In Table \ref{tab:transport:gender}, we can visualize ``gender-neutral" predictions, derived from the logistic regression (GLM), a boosting algorithm (GBM) and a random forest (RF). The first column corresponds to the proportional approach discussed in Section \ref{sec:5:wassertein:mitigate}.
\begin{table}[!h]
    \centering
    \begin{tabular}{|r|cccc|cccc|}\cline{2-9}
   \multicolumn{1}{c|}{} & \multicolumn{4}
   {c|}{\textcolor{couleurA}{\code{A}} (men)} & \multicolumn{4}
   {c|}{\textcolor{couleurB}{\code{B}} (women)} \\\cline{2-9}
    \multicolumn{1}{c|}{}& $\times 0.94$& GLM & GBM & RF & $\times 1.11$& GLM & GBM & RF \\\hline
        $m(\boldsymbol{x})=5\%$ & \textcolor{couleurA}{4.73\%}& \textcolor{couleurA}{4.94\%}& \textcolor{couleurA}{4.80\%}& \textcolor{couleurA}{4.42\%}& \textcolor{couleurB}{5.56\%}& \textcolor{couleurB}{5.16\%}& \textcolor{couleurB}{5.25\%}& \textcolor{couleurB}{6.15\%} \\
        $m(\boldsymbol{x})=10\%$& \textcolor{couleurA}{9.46\%} & \textcolor{couleurA}{9.83\%} & \textcolor{couleurA}{9.66\%} & \textcolor{couleurA}{8.92\%}& \textcolor{couleurB}{11.12\%} & \textcolor{couleurB}{10.38\%}& \textcolor{couleurB}{10.49\%}& \textcolor{couleurB}{12.80\%}  \\
        $m(\boldsymbol{x})=20\%$& \textcolor{couleurA}{18.91\%} & \textcolor{couleurA}{19.50\%} & \textcolor{couleurA}{18.68\%} & \textcolor{couleurA}{18.26\%}& \textcolor{couleurB}{22.25\%} & \textcolor{couleurB}{20.77\%}& \textcolor{couleurB}{21.63\%}& \textcolor{couleurB}{21.12\%}\\\hline
        % average prediction & \textcolor{couleurA}{9.05\%}& \textcolor{couleurA}{9.03\%}& \textcolor{couleurA}{8.73\%}& \textcolor{couleurB}{9.08\%}& \textcolor{couleurB}{9.06\%}& \textcolor{couleurB}{8.84\%}\\\hline
    \end{tabular}
    \caption{``Gender-free" prediction if the initial prediction was 5\% (on top), 10\% (in the middle) and 20\% (below). The first approach is the simple ``benchmark" based on $\mathbb{P}[Y=1]/\mathbb{P}[Y=1|S=s]$, and then three models are considered, GLM, GBM and RF.}
    \label{tab:transport:gender}
\end{table}
In Figures \ref{fig:distribution:bary:matching} and \ref{fig:distribution:score:scatter}, we have seen how to get a ``fair prediction", with the matching between $m(\boldsymbol{x},s=\textcolor{couleurA}{\code{A}})$ and $m^\star(\boldsymbol{x},s=\textcolor{couleurA}{\code{A}})$, on top, and between $m(\boldsymbol{x},s=\textcolor{couleurB}{\code{B}})$ and $m^\star(\boldsymbol{x},s=\textcolor{couleurB}{\code{B}})$, on Figure \ref{fig:distribution:bary:matching}, and with scatterplot of points $(m(\boldsymbol{x}_i,s_i=\textcolor{couleurA}{\code{A}}),m^\star(\boldsymbol{x}_i,s=\textcolor{couleurA}{\code{A}}))$ and $(m(\boldsymbol{x}_i,s_i=\textcolor{couleurB}{\code{B}}),m^\star(\boldsymbol{x}_i,s=\textcolor{couleurB}{\code{B}}))$ on Figure \ref{fig:distribution:score:scatter}.

\subsection{Age of the Main Driver}

Age is more complex variable. In insurance, age is usually considered ``{\em less discriminatory}" than gender, as we have seen, because as \cite{macnicol2006age} observes, age is not a club in which one enters at birth, and it will change with time. 
Age also seen as legitimate since it is is strongly correlated with inexperience, lack of skill, and risk-taking behaviors have been associated with the collisions of young drivers, \cite{rolison2018factors}. Though its use is not without discussion. For example, in Labrador (Canada), age cannot be used
before 55, and beyond that, it must be a discount (as in North Carolina, U.S.).

To illustrate the effect of non-discriminative predictions, we consider a binary sensitive attribute, related to the age, with $s=\boldsymbol{1}(\code{age}>65)$ (discrimination against old people), in Table \ref{tab:transport:old} and $s=\boldsymbol{1}(\code{age}<30)$ (discrimination against young people), in Table \ref{tab:transport:young}.

\begin{table}[!h]
    \centering
    \begin{tabular}{|r|cccc|cccc|}\cline{2-9}
   \multicolumn{1}{c|}{} & \multicolumn{4}
   {c|}{\textcolor{couleurA}{\code{A}} (younger $<65$)} & \multicolumn{4}
   {c|}{\textcolor{couleurB}{\code{B}} (old $>65$)} \\\cline{2-9}
    \multicolumn{1}{c|}{}& $\times 1.01$& GLM & GBM & RF & $\times 0.94$& GLM & GBM & RF \\\hline
        $m(\boldsymbol{x})=5\%$ & \textcolor{couleurA}{5.05\%}& \textcolor{couleurA}{5.17\%} & \textcolor{couleurA}{5.10\%} & \textcolor{couleurA}{5.27\%} & \textcolor{couleurB}{4.71\%} & \textcolor{couleurB}{3.84\%}& \textcolor{couleurB}{3.84\%}& \textcolor{couleurB}{3.96\%} \\
        $m(\boldsymbol{x})=10\%$ & \textcolor{couleurA}{10.09\%}& \textcolor{couleurA}{10.37\%} & \textcolor{couleurA}{10.16\%} & \textcolor{couleurA}{11.00\%}& \textcolor{couleurB}{9.42\%}  & \textcolor{couleurB}{7.81\%}& \textcolor{couleurB}{9.10\%}& \textcolor{couleurB}{6.88\%} \\
        $m(\boldsymbol{x})=20\%$ & \textcolor{couleurA}{20.19\%}& \textcolor{couleurA}{19.98\%} & \textcolor{couleurA}{19.65\%} & \textcolor{couleurA}{21.26\%} & \textcolor{couleurB}{18.85\%} & \textcolor{couleurB}{19.78\%}& \textcolor{couleurB}{23.79\%}& \textcolor{couleurB}{12.54\%}\\\hline
        % average prediction & \textcolor{couleurA}{9.05\%}& \textcolor{couleurA}{9.03\%}& \textcolor{couleurA}{8.73\%}& \textcolor{couleurB}{9.08\%}& \textcolor{couleurB}{9.06\%}& \textcolor{couleurB}{8.84\%}\\\hline
    \end{tabular}
    \caption{``Age-free" prediction (against old driver) if the initial prediction was 5\% (on top), 10\% (in the middle) and 20\% (below).}
    \label{tab:transport:old}
\end{table}

\begin{table}[!h]
    \centering
    \begin{tabular}{|r|cccc|cccc|}\cline{2-9}
   \multicolumn{1}{c|}{} & \multicolumn{4}
   {c|}{\textcolor{couleurA}{\code{A}} (young $<25$)} & \multicolumn{4}
   {c|}{\textcolor{couleurB}{\code{B}} (older $>25$)} \\\cline{2-9}
    \multicolumn{1}{c|}{}& $\times 0.74$& GLM & GBM & RF & $\times 1.06$& GLM & GBM & RF \\\hline
         $m(\boldsymbol{x})=5\%$ & \textcolor{couleurA}{3.71\%} & \textcolor{couleurA}{3.61\%} & \textcolor{couleurA}{4.45\%} & \textcolor{couleurA}{2.41\%} & \textcolor{couleurB}{5.29\%} & \textcolor{couleurB}{5.29\%}& \textcolor{couleurB}{5.14\%}& \textcolor{couleurB}{6.05\%} \\
        $m(\boldsymbol{x})=10\%$ & \textcolor{couleurA}{7.42\%} & \textcolor{couleurA}{7.89\%} & \textcolor{couleurA}{8.69\%} & \textcolor{couleurA}{5.17\%}& \textcolor{couleurB}{10.59\%} & \textcolor{couleurB}{10.29\%} & \textcolor{couleurB}{10.19\%}& \textcolor{couleurB}{11.95\%} \\
        $m(\boldsymbol{x})=20\%$ & \textcolor{couleurA}{14.84\%} & \textcolor{couleurA}{21.82\%} & \textcolor{couleurA}{18.09\%} & \textcolor{couleurA}{9.93\%} & \textcolor{couleurB}{21.17\%} & \textcolor{couleurB}{19.87\%}& \textcolor{couleurB}{20.33\%}& \textcolor{couleurB}{21.29\%}\\\hline
    \end{tabular}
    \caption{``Age-free" (against young drivers) prediction if the initial prediction was 5\% (on top), 10\% (in the middle) and 20\% (below).}
    \label{tab:transport:young}
\end{table}
% is considered as a sensitive attribute in many places COMPLETER. In our dataset, we have XXXX \% men and XXXX\% women registered as ``main driver".
In Figures  \ref{fig:distribution:bary:matching:old}and \ref{fig:distribution:bary:matching:young} we visualize the matchings between $m(\boldsymbol{x},s=\textcolor{couleurA}{\code{A}})$ and $m^\star(\boldsymbol{x},s=\textcolor{couleurA}{\code{A}})$, on top, and between $m(\boldsymbol{x},s=\textcolor{couleurB}{\code{B}})$ and $m^\star(\boldsymbol{x},s=\textcolor{couleurB}{\code{B}})$ below, respectively with $s=\boldsymbol{1}(\code{age}>65)$ (discrimination against old people) and $s=\boldsymbol{1}(\code{age}<30)$ (discrimination against young people).

\begin{figure}[!h]
    \centering
    \includegraphics[width=.32\textwidth]{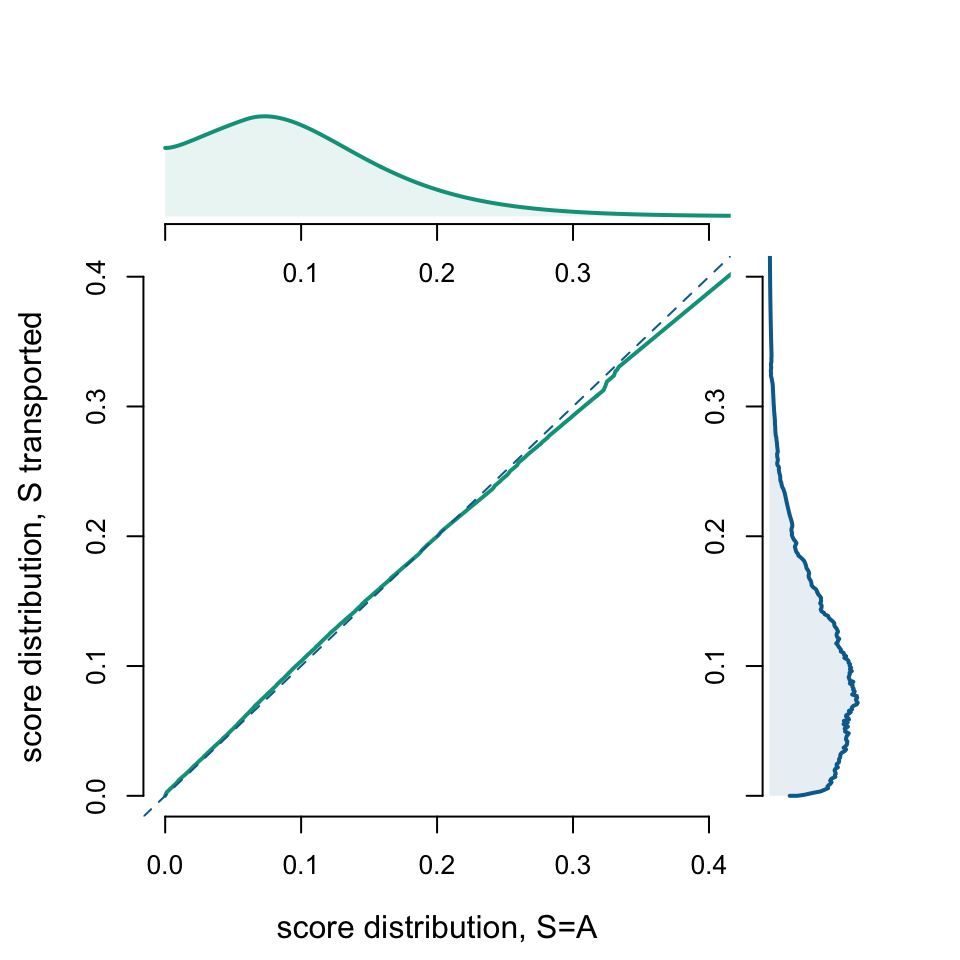} \includegraphics[width=.32\textwidth]{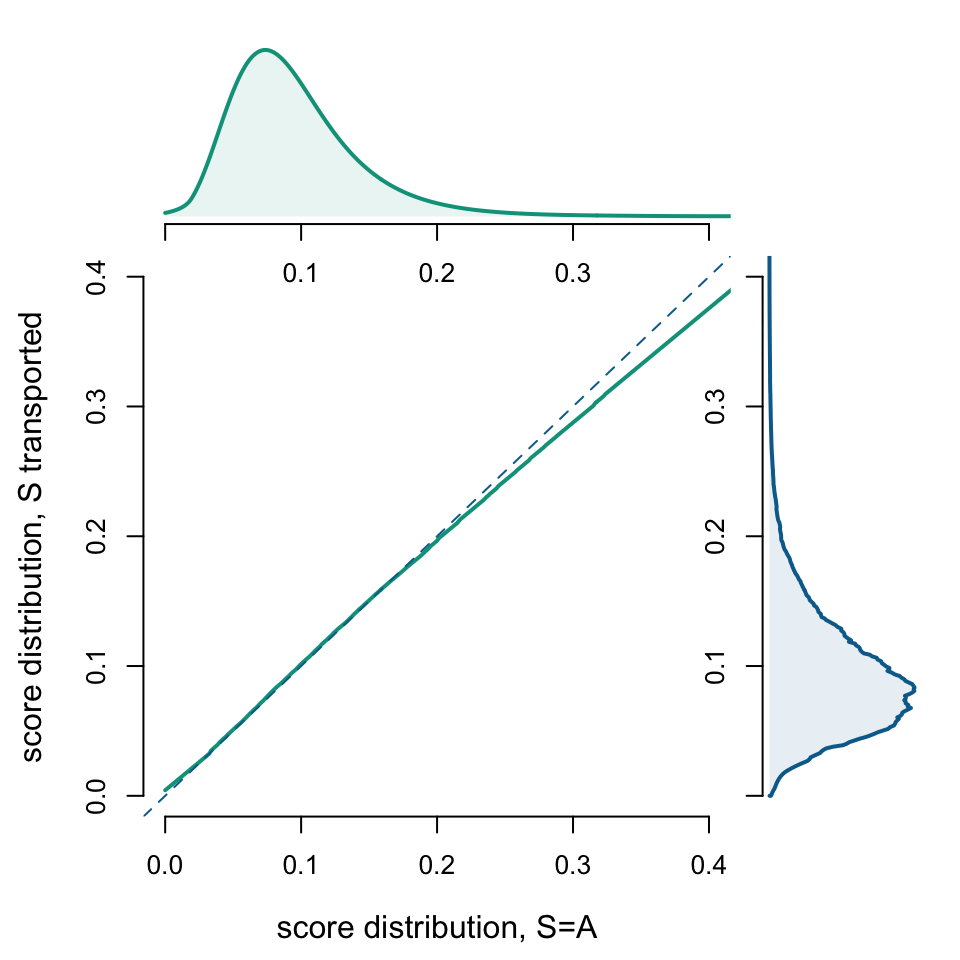} \includegraphics[width=.32\textwidth]{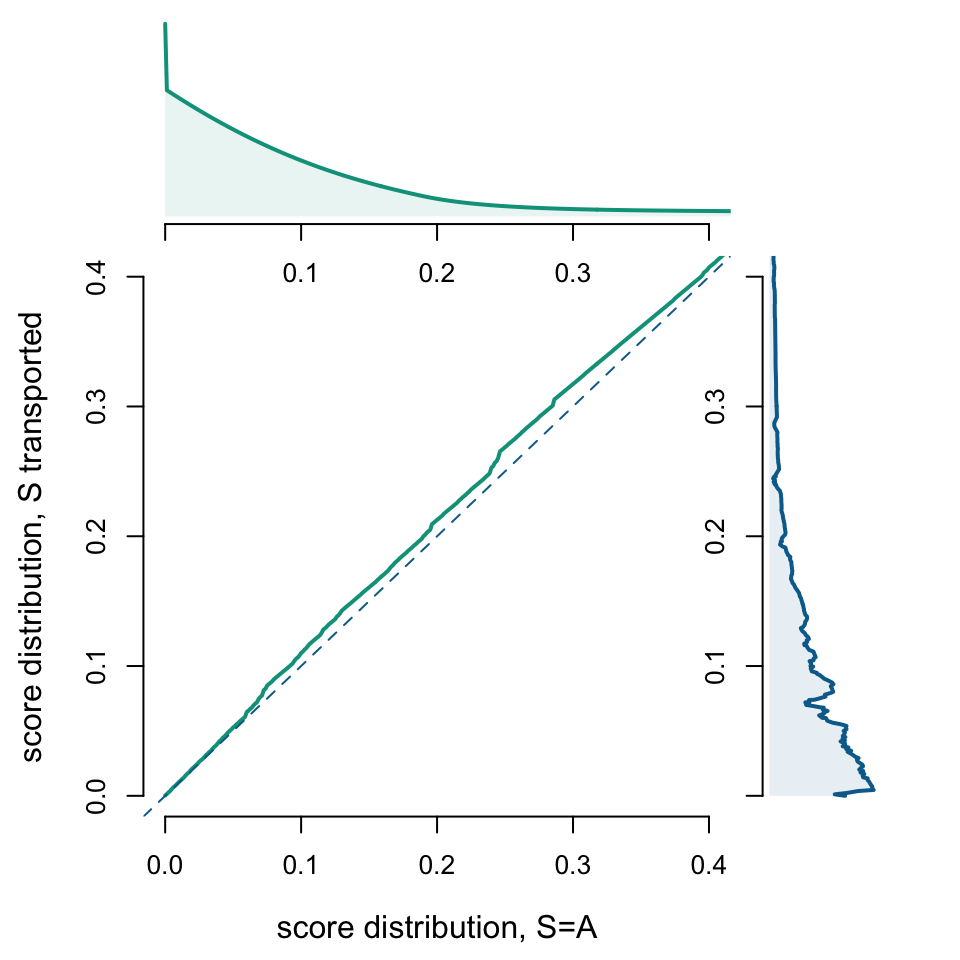}

    \includegraphics[width=.32\textwidth]{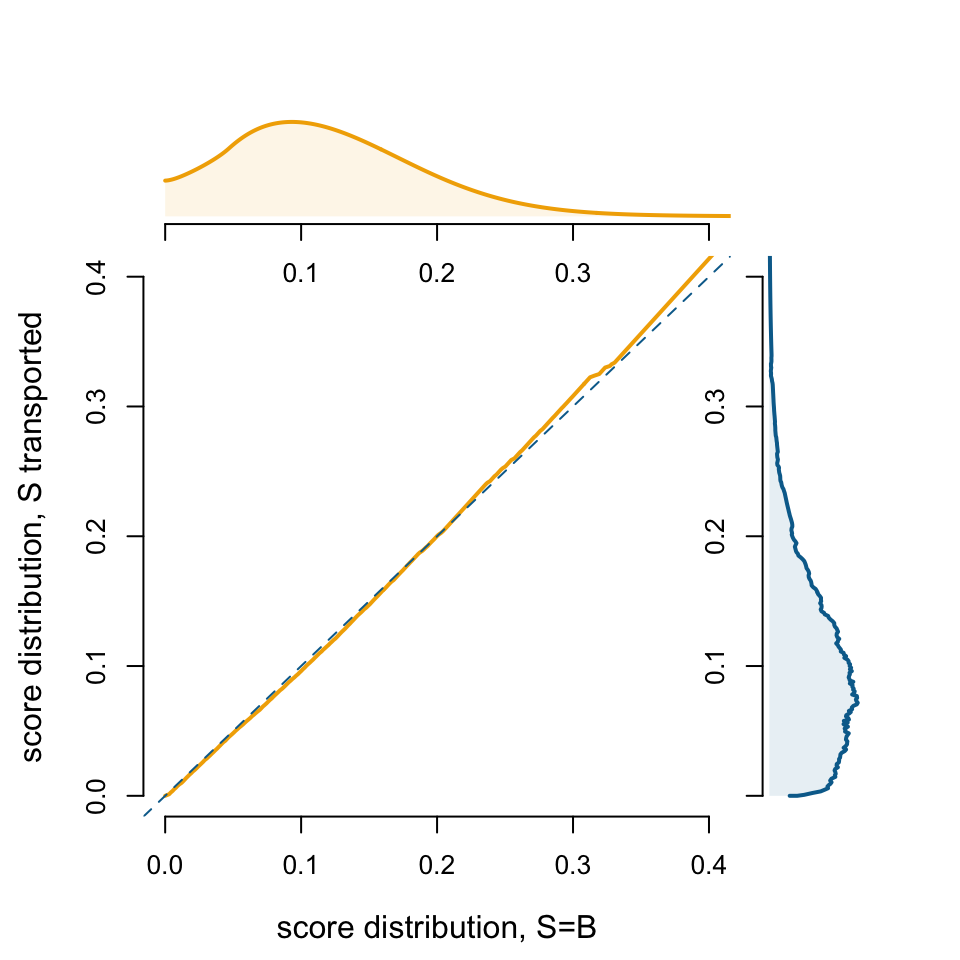} \includegraphics[width=.32\textwidth]{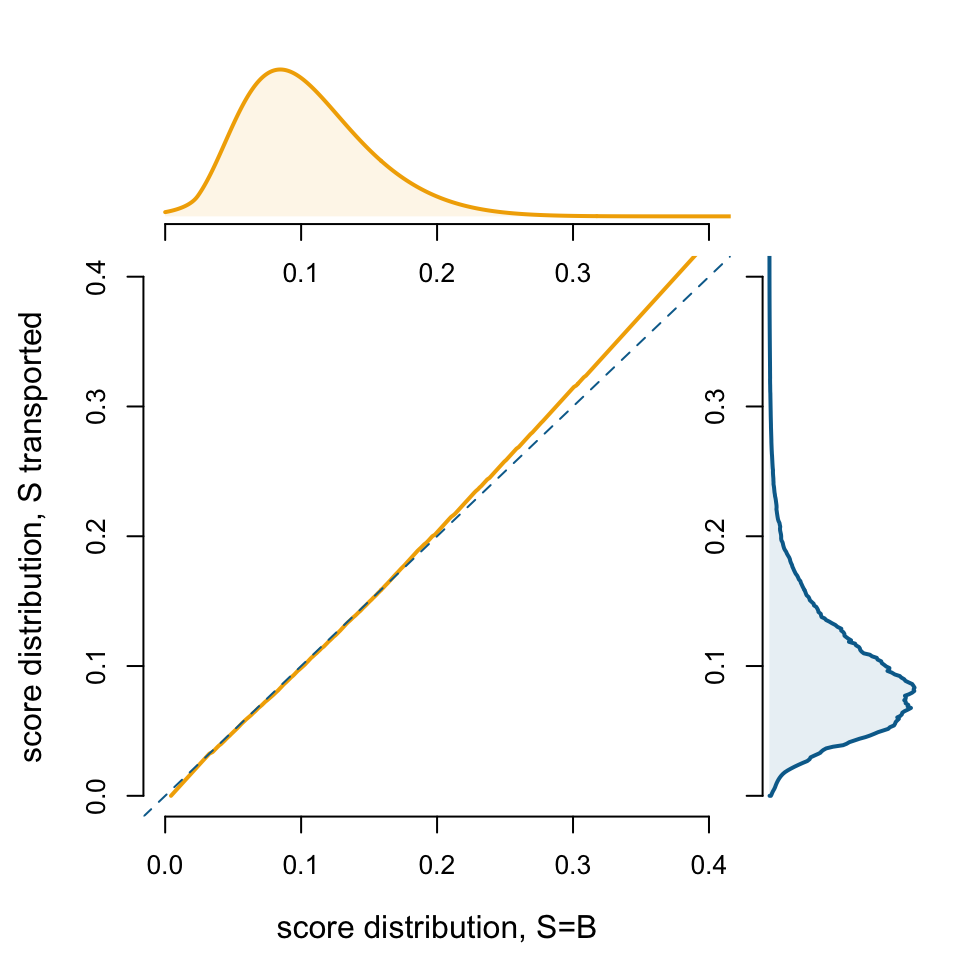} \includegraphics[width=.32\textwidth]{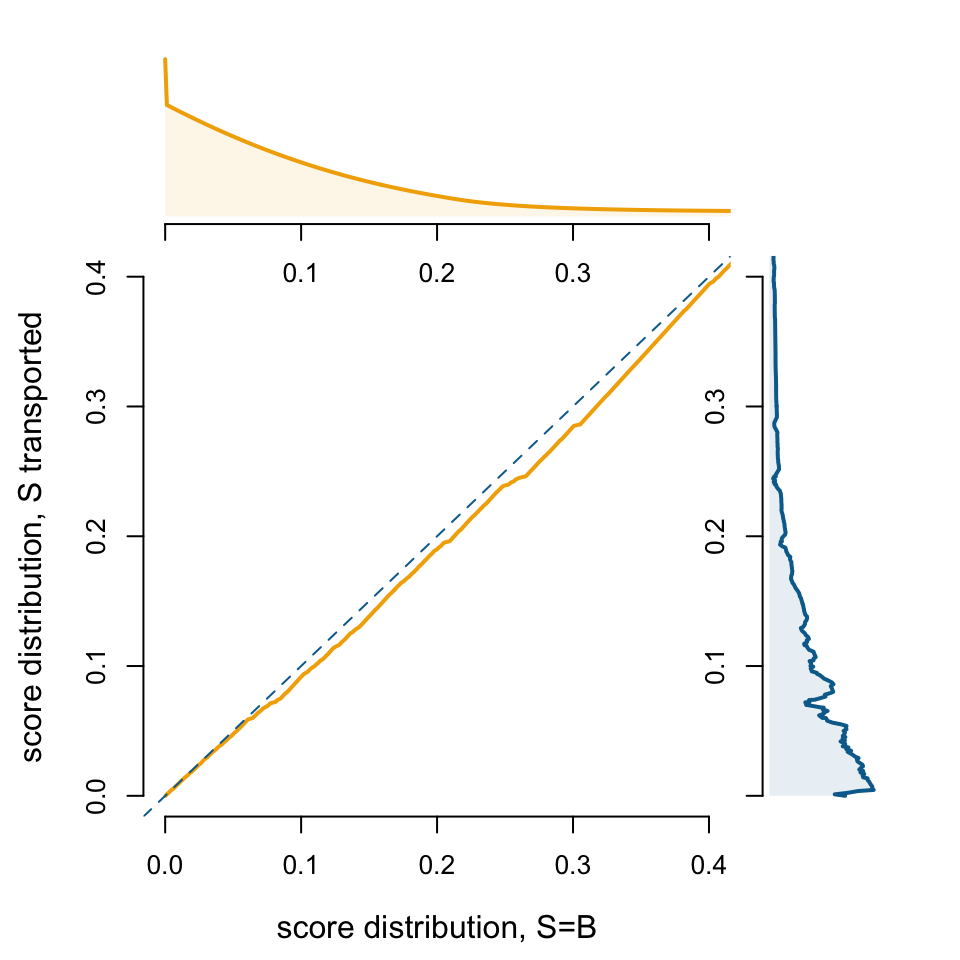}
    \caption{Matching between $m(\boldsymbol{x},s=\textcolor{couleurA}{\code{A}})$ and $m^\star(\boldsymbol{x},s=\textcolor{couleurA}{\code{A}})$, on top, and between $m(\boldsymbol{x},s=\textcolor{couleurB}{\code{B}})$ and $m^\star(\boldsymbol{x},s=\textcolor{couleurB}{\code{B}})$, below, on the probability to claim a loss in motor insurance when $s$ is the indicator that the driver is ``old" $\boldsymbol{1}(\code{age}>65)$.}
    \label{fig:distribution:bary:matching:old}
\end{figure}

% On Figure \ref{fig:distribution:score:scatter:old}

% \begin{figure}[!h]
%     \centering
%     \includegraphics[width=\textwidth]{figs/old/unnamed-chunk-18-1.png}
%     \caption{Scatterplot of points $(m(\boldsymbol{x}_i,s_i=\textcolor{couleurA}{\code{A}}),m(^\star\boldsymbol{x}_i))$ and $(m(\boldsymbol{x}_i,s_i=\textcolor{couleurB}{\code{B}}),m(^\star\boldsymbol{x}_i))$, with three models (GLM, GBM, RF), on the probability to claim a loss in motor insurance when $s$ is the indicator that the driver is ``old" $\boldsymbol{1}(\code{age}>65)$.}
%     \label{fig:distribution:score:scatter:old}
% \end{figure}

\begin{figure}[!h]
    \centering
    \includegraphics[width=.32\textwidth]{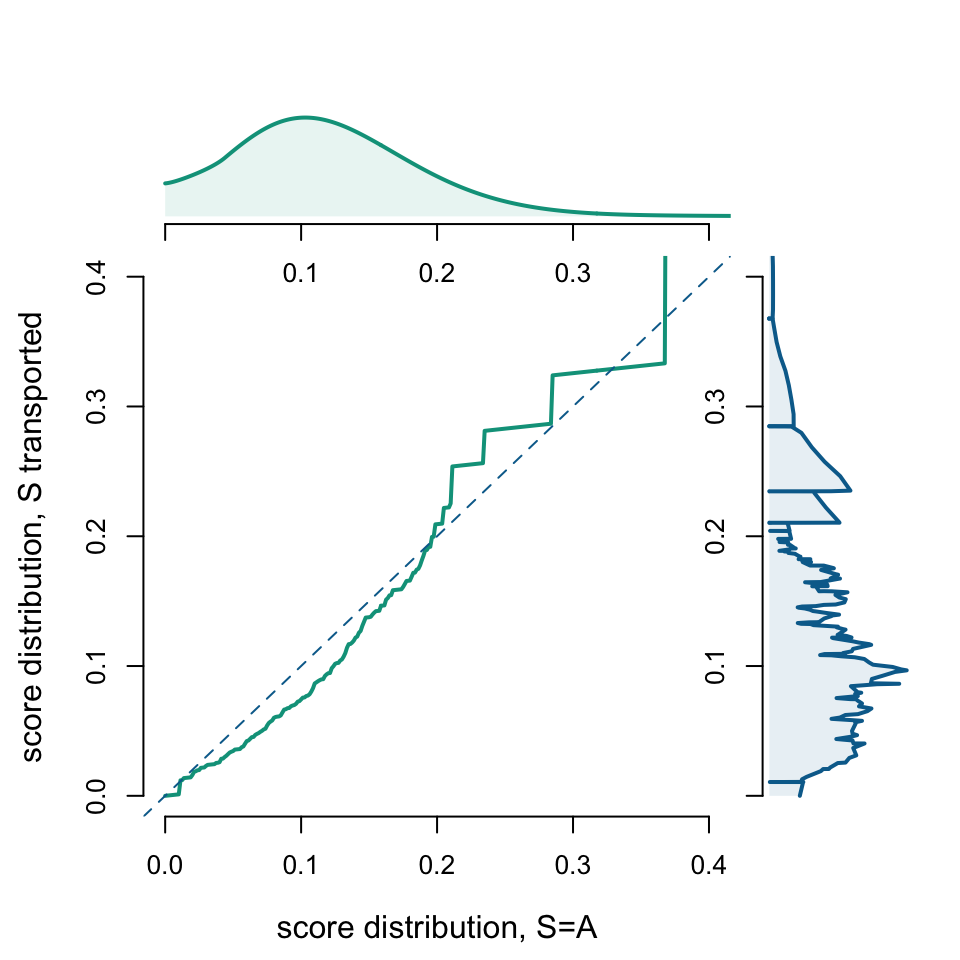} \includegraphics[width=.32\textwidth]{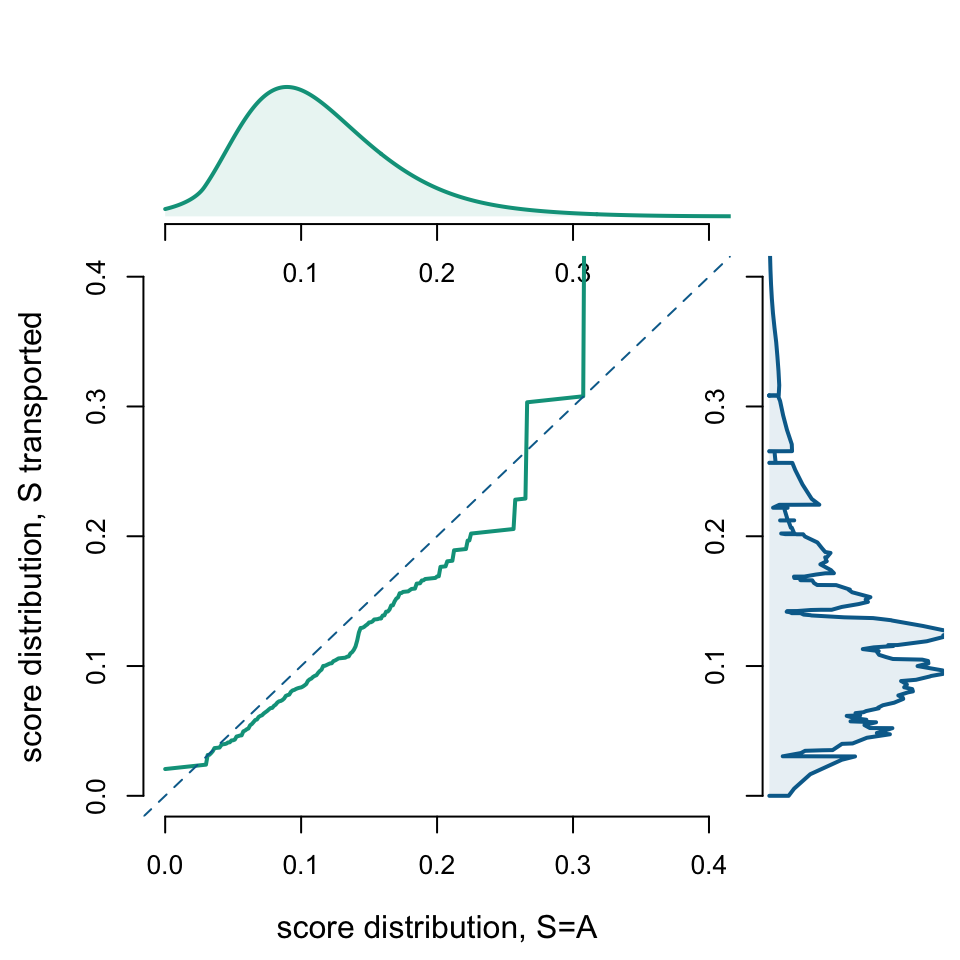} \includegraphics[width=.32\textwidth]{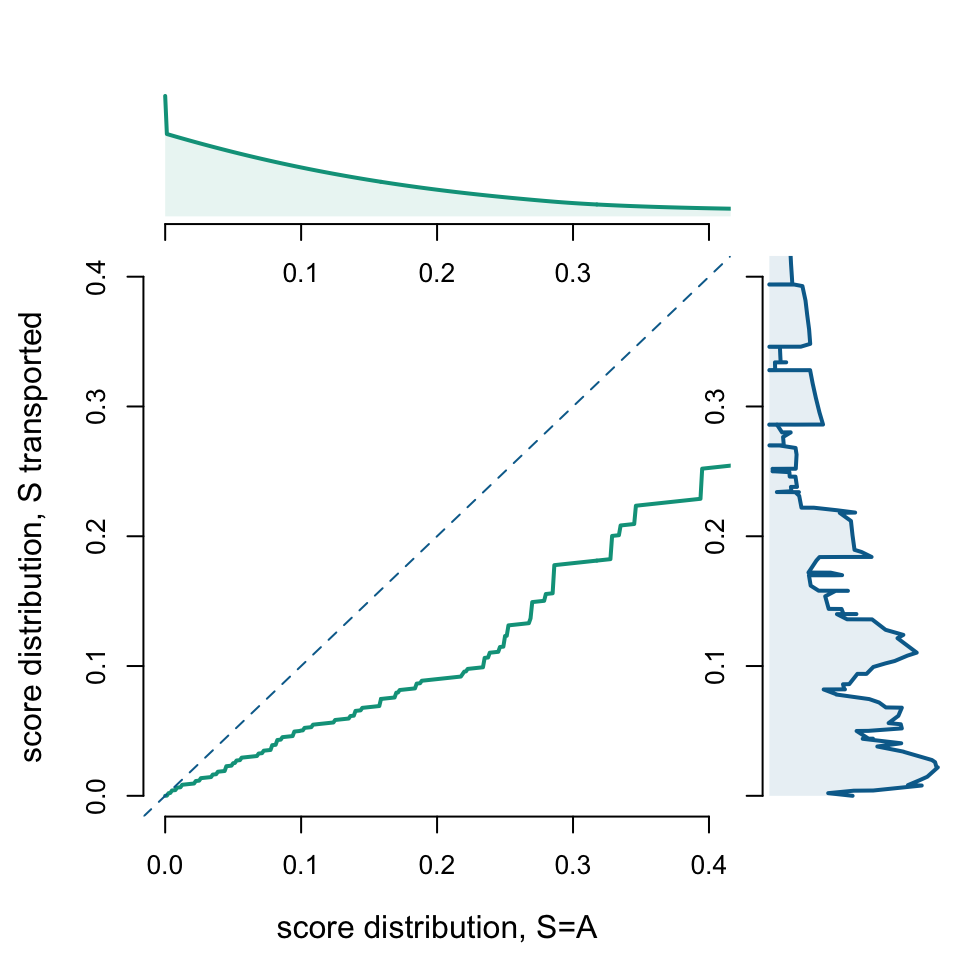}

    \includegraphics[width=.32\textwidth]{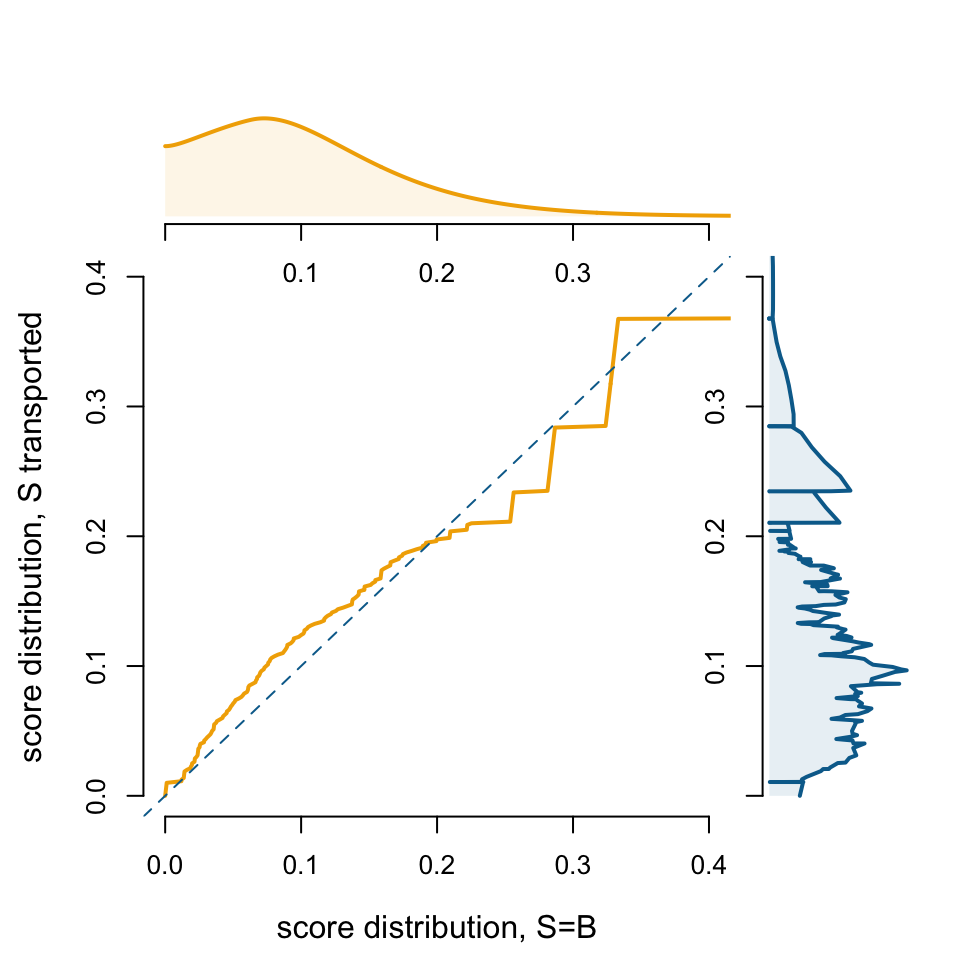} \includegraphics[width=.32\textwidth]{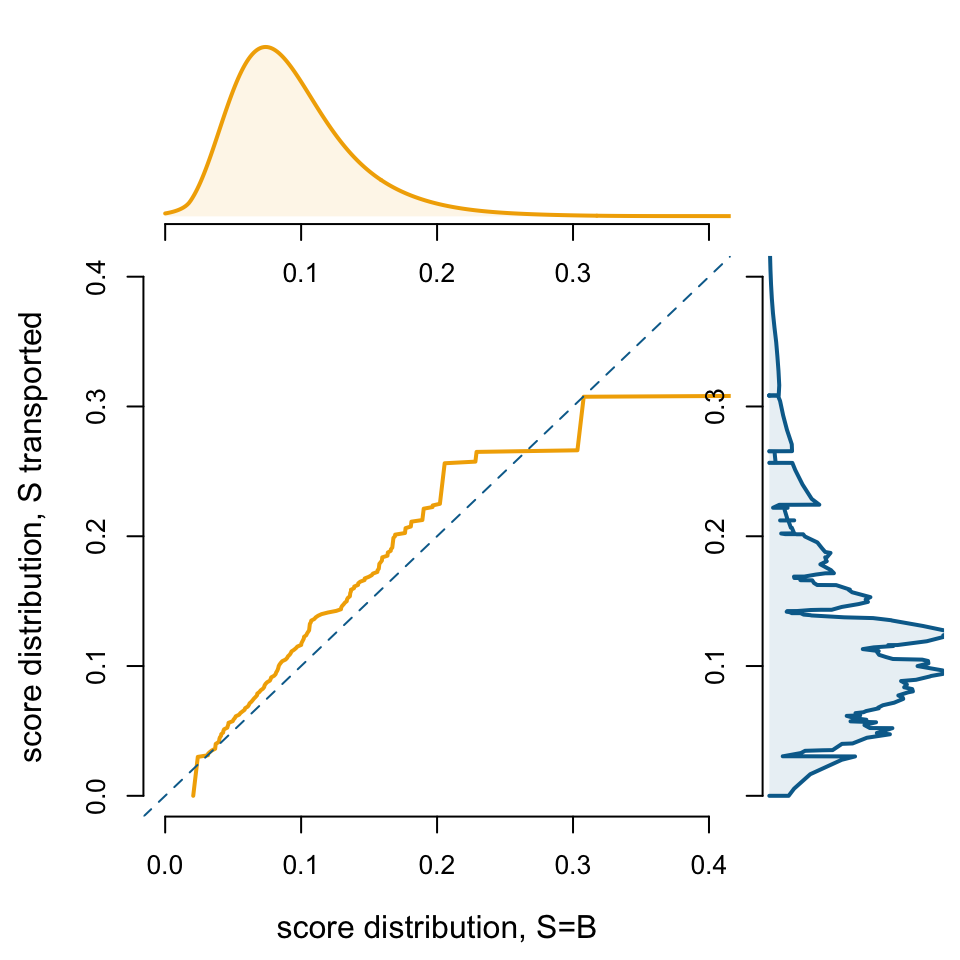} \includegraphics[width=.32\textwidth]{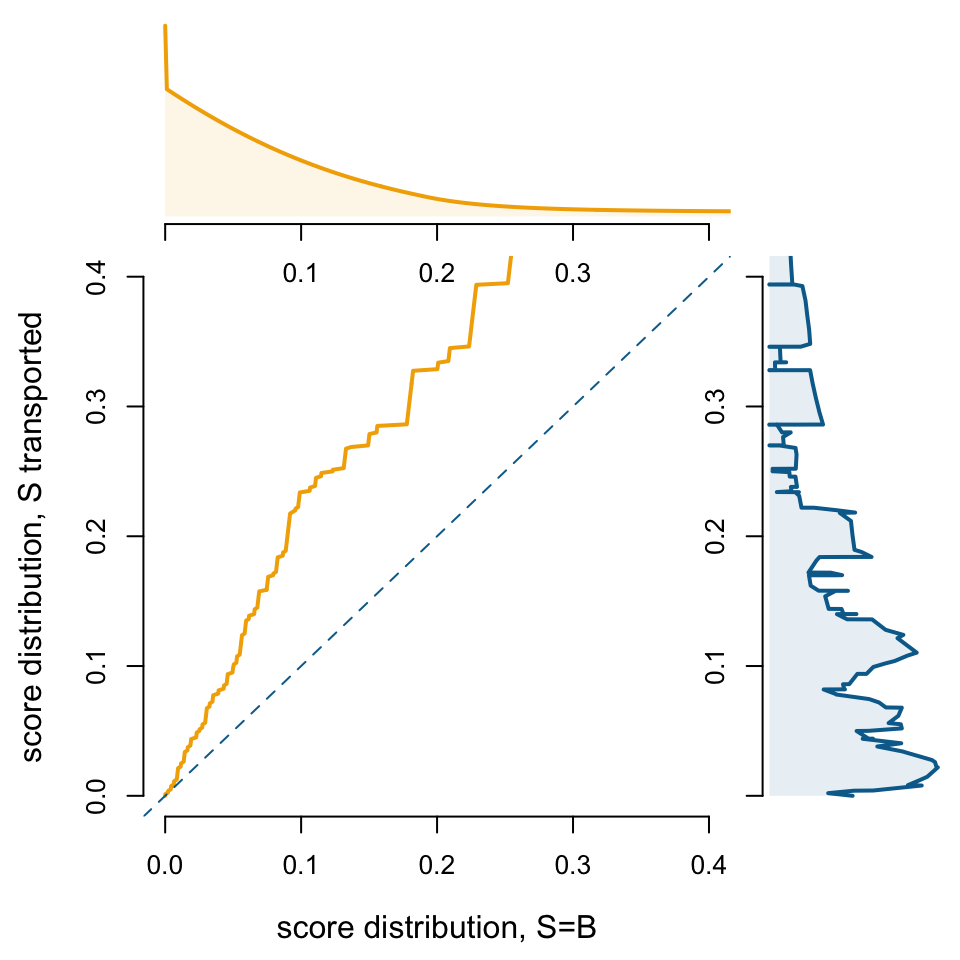}
    \caption{Matching between $m(\boldsymbol{x},s=\textcolor{couleurA}{\code{A}})$ and $m^\star(\boldsymbol{x},s=\textcolor{couleurA}{\code{A}})$, on top, and between $m(\boldsymbol{x},s=\textcolor{couleurB}{\code{B}})$ and $m^\star(\boldsymbol{x},s=\textcolor{couleurB}{\code{B}})$, below, on the probability to claim a loss in motor insurance when $s$  is the indicator that the driver is ``young" $\boldsymbol{1}(\code{age}<30)$.}
    \label{fig:distribution:bary:matching:young}
\end{figure}

% On Figure \ref{fig:distribution:score:scatter:young}

% \begin{figure}[!h]
%     \centering
%     \includegraphics[width=\textwidth]{figs/young/unnamed-chunk-18-1.png}
%     \caption{Scatterplot of points $(m(\boldsymbol{x}_i,s_i=\textcolor{couleurA}{\code{A}}),m(^\star\boldsymbol{x}_i))$ and $(m(\boldsymbol{x}_i,s_i=\textcolor{couleurB}{\code{B}}),m(^\star\boldsymbol{x}_i))$, with three models (GLM, GBM, RF), on the probability to claim a loss in motor insurance when $s$ is the indicator that the driver is ``young" $\boldsymbol{1}(\code{age}<30)$.}
%     \label{fig:distribution:score:scatter:young}
% \end{figure}

% young / non-young, $\boldsymbol{1}(age>t)$, $t$ = threshold

% influence of $t$ on the barycenter distribution.

\section{Conclusion}

We illustrated how discrimination naturally arises when models are used to predict risk based on a set of characteristics. Whereas some forms of discrimination can have legitimate reasons, they are often heavily correlated with sensitive attributes such as gender or race. Several notions of fairness and indeed several procedures to achieve fair predictions exist. We showed that the Wasserstein distance can be an effective tool to achieve fair predictions while employing the notion of optimal transport. This enables to take into account differences in the whole distribution of predictions across different groups instead of just shifting its mean, as a simple rescaling would.  The empirical results highlight the ease of the interpretation and value of the approach in promoting fair decision-making in the insurance industry. 

\bibliographystyle{plainnat}
\bibliography{biblio}

\end{document}